\documentclass{article}

\PassOptionsToPackage{numbers, compress}{natbib}

\usepackage[preprint]{neurips_2026}

\usepackage[utf8]{inputenc}
\usepackage[T1]{fontenc}
\usepackage{hyperref}
\usepackage{url}
\usepackage{booktabs}
\usepackage{graphicx}
\usepackage{amsfonts}
\usepackage{nicefrac}
\usepackage{microtype}
\usepackage{xcolor}
\usepackage{amsmath,amssymb,amsthm}
\usepackage{mathtools}
\usepackage{adjustbox}
\usepackage{enumitem}
\usepackage[most]{tcolorbox}

\theoremstyle{plain}
\newtheorem{proposition}{Proposition}
\newtheorem{lemma}[proposition]{Lemma}
\newtheorem{corollary}[proposition]{Corollary}

\theoremstyle{remark}

\title{APEX: Amplitude Anchors and Phase Priors for Target-Scarce Higher-Frequency Wave Prediction}

\author{
Yifan Sun \quad Lei Cheng \quad Sijie Chen \quad Ting Zhang \quad Jianlong Li \quad Shikai Fang\\[3pt]
College of Information Science and Electronic Engineering, Zhejiang University
}

\begin{document}

\maketitle

\begin{abstract}
Learning-based surrogates have become increasingly effective for wave-field prediction, and neural operators in particular have shown strong performance within observed frequency regimes. 
However, higher-frequency prediction under scarce target supervision remains comparatively underexplored, especially in wave problems where higher-frequency data are substantially more expensive to simulate or measure than lower-frequency data. 
A central difficulty is that cross-frequency transfer is inherently asymmetric: coarse amplitude structure remains relatively stable across frequencies, whereas phase-sensitive oscillatory structure deteriorates much more rapidly as frequency increases. 
Motivated by this asymmetry, we propose \textbf{APEX}, \underline{\textbf{A}}mplitude-anchored and \underline{\textbf{P}}hase-prior-guided \underline{\textbf{E}}nhancement from e\underline{\textbf{X}}trapolated coarse predictions,
a framework for target-scarce higher-frequency wave-field prediction.
A lower-frequency neural operator first provides a coarse prediction in the target-frequency regime, from which we retain only the amplitude as a transferable structural anchor. A conditional flow-matching enhancer then reconstructs the target higher-frequency field under the guidance of a Green's-function-inspired phase prior. Experiments on SimpleWave, Helmholtz, and Maxwell benchmarks show that APEX consistently outperforms direct lower-to-higher extrapolation, target-adapted operator, and joint generative baselines under limited target-frequency supervision. Our results suggest that reliable higher-frequency prediction of oscillatory wave fields should not rely on direct end-to-end transfer of the full complex field, but instead on explicitly reusing transferable coarse structure while separately recovering the missing oscillatory detail.
\end{abstract}

\section{Introduction}

Learning-based surrogates have become a powerful paradigm for approximating solution maps of parameterized PDEs, with broad impact across fluid dynamics, electromagnetics, acoustics, and other scientific domains~\cite{gu2022neurolight,zou2024deep,hao2023gnot}. 
Among these approaches, neural operators have emerged as a particularly effective class of models for learning nonlocal mappings in function space, making them a natural backbone for wave-field prediction~\cite{lifourier,kovachki2023neural,lu2021learning}. 
More broadly, this places wave-field modeling in the growing problem of \emph{cross-frequency generalization} for learned surrogates: models are expected not only to interpolate within observed spectral regimes, but also to remain reliable as the queried frequency changes. 
In many practical wave problems, this challenge becomes especially important at higher frequencies, which are substantially more expensive to simulate or measure~\cite{IHLENBURG19959}.

Despite rapid progress in learned PDE surrogates and neural operators, this higher-frequency regime remains comparatively underexplored. 
Most existing studies on cross-frequency wave prediction focus on interpolation within observed spectral ranges, adaptation with target-regime supervision, or general in-range surrogate accuracy~\cite{gu2022neurolight,seo2024wave,sun2025hankel,wang2024transfer}. 
In this paper, we focus on a particularly important regime of this broader problem: \emph{higher-frequency prediction under scarce target supervision}. 
This setting arises naturally in wave problems because higher-frequency samples typically require finer discretization, stricter numerical control, and substantially higher simulation or measurement cost~\cite{IHLENBURG19959}. 
As a result, practitioners often have abundant lower-frequency data but only limited supervision in the higher-frequency target regime.

A straightforward strategy in this setting is to train a surrogate on abundant lower-frequency data and directly apply it at unseen higher frequencies, i.e., \emph{direct lower-to-higher extrapolation}. 
Another common strategy is to adapt the lower-frequency model using the limited target-frequency data that can be afforded~\cite{sun2025hankel,wang2024transfer}. 
In practice, however, both approaches often remain unreliable in highly oscillatory higher-frequency regimes, where wave responses are typically more sensitive to propagation mismatch~\cite{Keller1962diffraction,sethian1996afast}. 
More importantly, what remains insufficiently understood is not merely whether lower-to-higher prediction degrades, but \emph{which} components of the complex field remain transferable and \emph{which} ones fail first under frequency shift.
This missing mechanism-level understanding makes it difficult to design models that generalize reliably beyond the observed regime.

In this work, we study higher-frequency prediction under scarce target supervision through the lens of transferability across field components. 
Our central finding is that cross-frequency transfer is inherently asymmetric: coarse amplitude structure remains substantially more stable across frequencies, while phase-sensitive oscillatory structure deteriorates much more rapidly as the frequency gap increases. 
This asymmetry appears both in the ground-truth cross-frequency similarity structure and in the outputs of neural operators trained on lower frequencies and directly queried at higher frequencies. 
Motivated by this finding, we propose \textbf{APEX}, \underline{\textbf{A}}mplitude-anchored and \underline{\textbf{P}}hase-prior-guided \underline{\textbf{E}}nhancement from e\underline{\textbf{X}}trapolated coarse predictions,
a framework for target-scarce higher-frequency wave-field prediction.
APEX uses a frozen lower-frequency Fourier neural operator (FNO) to provide a coarse amplitude anchor, and then reconstructs the target higher-frequency field using a conditional flow-matching enhancer~\cite{lipman2023flow,tong2024improving} guided by a Green's-function-inspired phase prior. 
Experiments on SimpleWave, Helmholtz, and Maxwell benchmarks show that APEX consistently outperforms direct lower-to-higher extrapolation, target-adapted operator, and joint generative baselines under limited target-frequency supervision~\cite{wang2024transfer,sun2025hankel,huang2024diffusionpde}. 
More broadly, our results suggest that reliable higher-frequency prediction should not rely on direct end-to-end transfer of the full complex field, but instead on explicitly reusing transferable coarse structure while separately recovering the missing oscillatory detail.

Our contributions are summarized as follows:
\begin{itemize}
	\item We identify a key \emph{transfer asymmetry} in higher-frequency wave-field prediction: coarse amplitude remains substantially more stable across frequencies than phase-sensitive oscillatory structure.
	\item Based on this finding, we propose \textbf{APEX}, a framework that decomposes higher-frequency prediction under scarce target supervision into coarse structural transfer and oscillatory reconstruction.
	\item Experiments on SimpleWave, Helmholtz, and Maxwell benchmarks show that APEX consistently improves higher-frequency prediction over direct lower-to-higher extrapolation, target-adapted operator, and joint generative baselines.
\end{itemize}
\section{Why Direct Higher-Frequency Extrapolation Fails}

\subsection{Wave Equation and Green’s Function}

We consider parameterized frequency-domain wave equations of the form
\begin{equation}
	\mathcal{L}_{e,\nu}[u](\mathbf r)=s(\mathbf r;e,\nu),
	\qquad \mathbf r\in\Omega,
	\label{eq:general_pde}
\end{equation}
where $\mathbf r$ denotes spatial coordinates, $\Omega$ is the spatial domain, $s(\mathbf r;e,\nu)$ is the source term, $e$ denotes the environment and geometry, $\nu$ is the spectral variable, and $u$ is the resulting complex-valued field. Throughout this work, $\nu$ is taken to be positive, corresponding to a physical frequency, wavenumber, or dimensionless spectral parameter depending on the benchmark. This form covers standard frequency-domain wave models, including Maxwell and Helmholtz equations \cite{colton2013inverse}, as well as the SimpleWave simulator used in this work.

For linear frequency-domain problems, the field admits a Green's-function representation of the form~\cite{colton2013inverse,stakgold2011greens}
\begin{equation}
	u(\mathbf r;e,\nu)
	=
	\int_{\Omega} G_{e,\nu}(\mathbf r,\mathbf r')\,s(\mathbf r';e,\nu)\,d\mathbf r',
	\label{eq:general_green}
\end{equation}
where $G_{e,\nu}(\mathbf r,\mathbf r')$ is the associated Green's function. Eq.~\eqref{eq:general_green} is introduced here not to develop a full Green's-function analysis, but to emphasize a structural fact that will matter later: frequency-domain wave fields are organized by source-dependent propagation geometry. This viewpoint motivates why, in the higher-frequency regime, geometric oscillatory cues can be useful even when a full direct surrogate prediction becomes unreliable.

\begin{figure}[t]
	\centering
	\adjustbox{max width=0.98\textwidth}{%
		\shortstack[c]{%
			\includegraphics[height=0.26\textwidth,keepaspectratio]{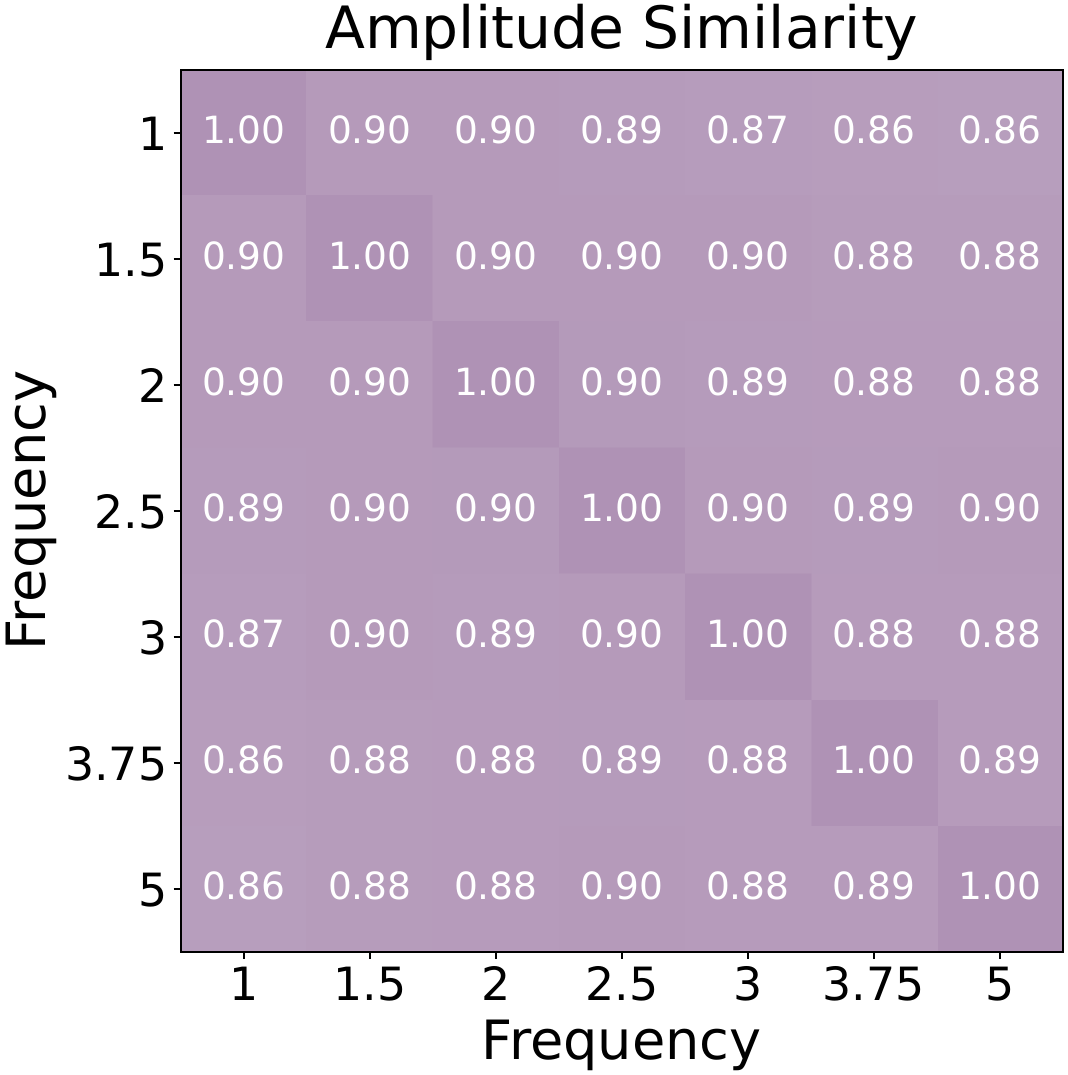}\\[-0.2ex]
			\hspace*{1.5em}(a)}
		\hspace{0.005\textwidth}
		\shortstack[c]{%
			\includegraphics[height=0.26\textwidth,keepaspectratio]{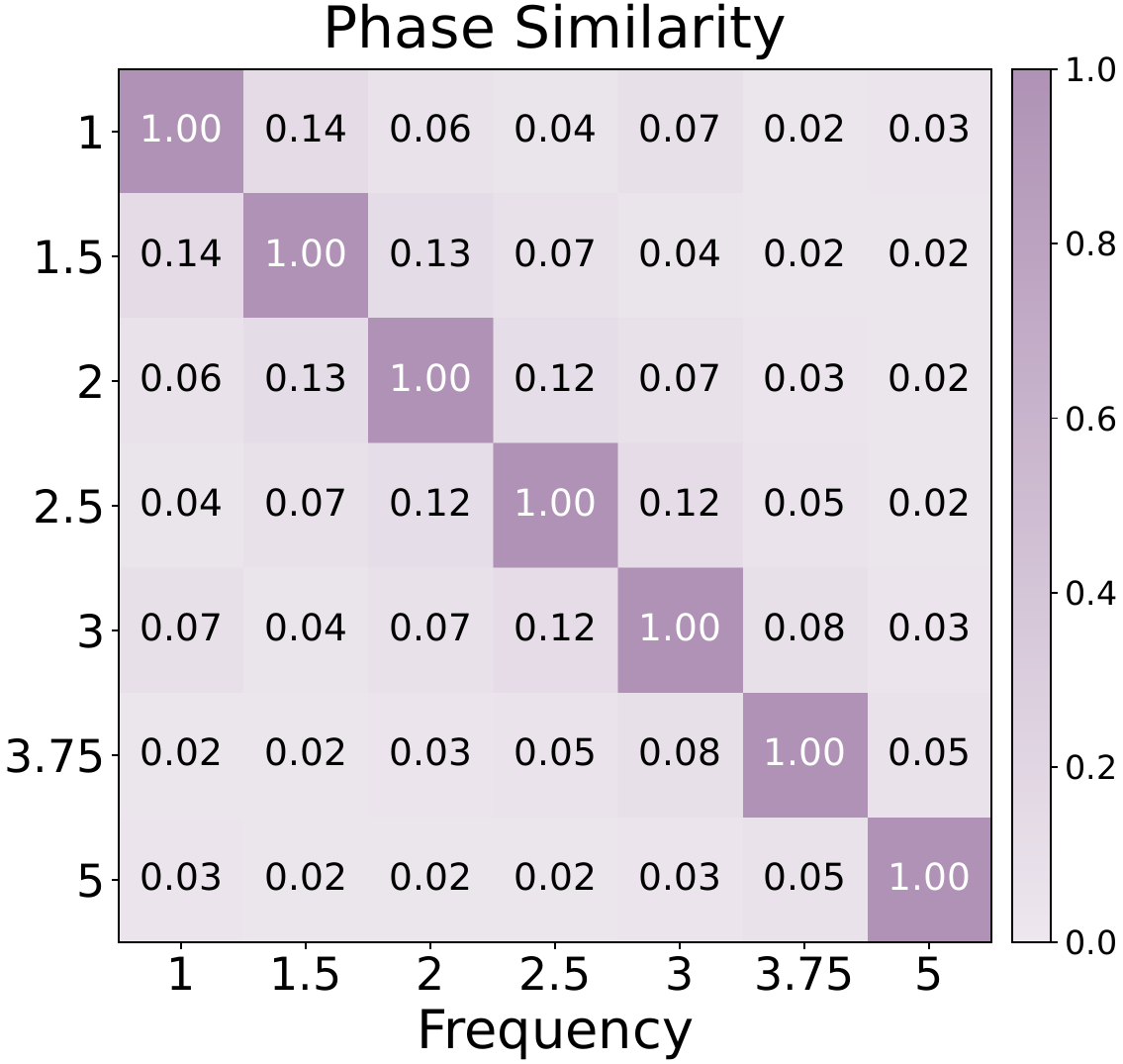}\\[-0.2ex]
			(b)}
		\hspace{0.01\textwidth}
		\shortstack[c]{%
			\includegraphics[height=0.26\textwidth,keepaspectratio]{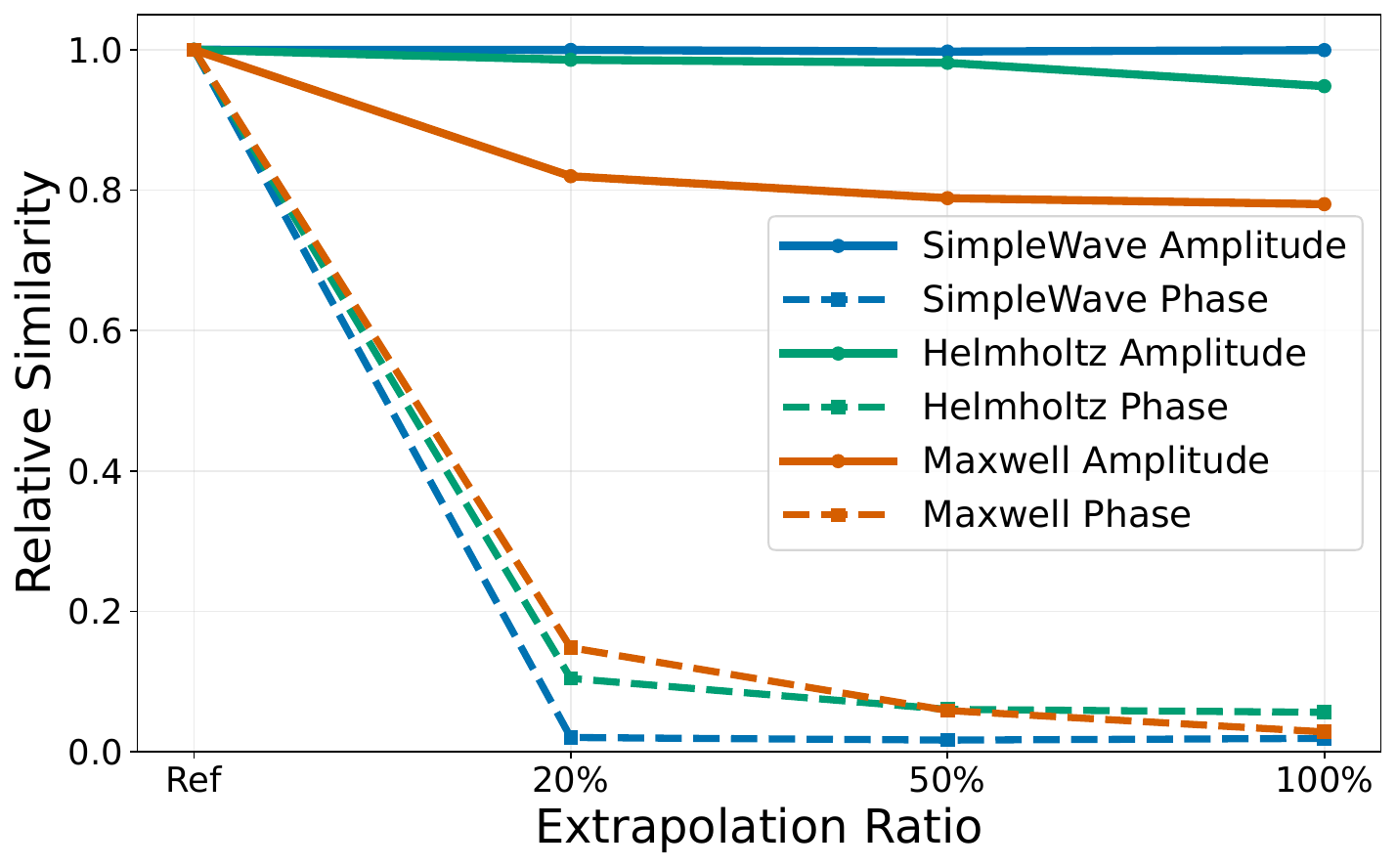}\\[-0.2ex]
			\hspace*{1.5em}(c)}
	}
	\caption{Empirical evidence for cross-frequency transfer asymmetry. (a,b) Ground-truth amplitude and phase similarity on the Maxwell benchmark. (c) Relative amplitude and phase similarity of lower-frequency FNO extrapolation across the three benchmarks.}
	\label{fig:em_similarity}
\end{figure}

\subsection{Neural Operators and Direct Lower-to-Higher Extrapolation}
\label{sec:2.2}

Neural operators provide surrogate models for PDE solution operators.
A common neural-operator layer can be written as~\cite{kovachki2023neural,lifourier}
\begin{equation}
	v_{\ell+1}(\mathbf r)
	=
	\sigma\!\left(
	W_\ell v_\ell(\mathbf r)
	+
	\int_{\Omega}\kappa_\ell(\mathbf r,\mathbf r';e,\nu)\,v_\ell(\mathbf r')\,d\mathbf r'
	\right),
	\label{eq:operator_kernel_layer}
\end{equation}
where $v_\ell$ denotes the latent field at layer $\ell$, $W_\ell$ is a linear transformation, $\sigma(\cdot)$ is a nonlinear activation function, and $\kappa_\ell$ parameterizes an integral kernel.

In the cross-frequency setting, let $\mathcal V_{\mathrm{train}}=\{\nu_1,\ldots,\nu_m\}$ denote the observed training frequencies, and let a neural operator $F_\theta$ be learned from samples $\{(e, \nu, u(\mathbf r;e,\nu)):\nu\in\mathcal V_{\mathrm{train}}\}$. When the trained model is queried at a new frequency $\tilde\nu\notin\mathcal V_{\mathrm{train}}$,
\begin{equation}
	\hat u(\mathbf r;e,\tilde\nu)=F_\theta(e,\tilde\nu)(\mathbf r),
	\label{eq:cross_freq_operator}
\end{equation}
the problem becomes \emph{direct lower-to-higher extrapolation} if $\tilde\nu$ lies outside the observed frequency range. This is the naive strategy studied in this paper: train on lower frequencies and directly query the learned operator at higher ones. Existing work has mainly focused on interpolation or in-range prediction~\cite{gu2022neurolight,seo2024wave}, whereas this out-of-range higher-frequency regime remains much less explored. The key question is therefore not only whether the prediction degrades, but also \emph{how} that degradation is distributed across the complex field representation.

\subsection{The Asymmetry of Cross-Frequency Transferability}
\label{sec:asymmetry_transferability}

We begin from an empirical observation: cross-frequency variation does not affect all components of a complex wave field equally. To visualize this effect, we first compare the same Maxwell sample across different frequencies using separate similarity measures for amplitude and phase. Let $A(\mathbf r;\nu)$ and $\phi(\mathbf r;\nu)$ denote the amplitude and phase of the field, respectively. For two frequencies $\nu_i$ and $\nu_j$, we measure amplitude similarity by
\begin{equation}
	S_A(\nu_i,\nu_j)
	=
	\frac{\left\langle A(\mathbf r;\nu_i),A(\mathbf r;\nu_j)\right\rangle}
	{\|A(\mathbf r;\nu_i)\|_2\,\|A(\mathbf r;\nu_j)\|_2},
\end{equation}
and phase similarity by the phase-factor coherence
\begin{equation}
	S_P(\nu_i,\nu_j)
	=
	\left|
	\frac{1}{N}
	\sum_{\mathbf r\in\Omega}
	\exp\!\left(i\big(\phi(\mathbf r;\nu_i)-\phi(\mathbf r;\nu_j)\big)\right)
	\right|.
\end{equation}
The resulting ground-truth similarity maps are shown in Fig.~\ref{fig:em_similarity}(a)--(b). They reveal a clear asymmetry: amplitude remains comparatively structured across frequency pairs, whereas phase coherence decays much more rapidly as the frequency gap increases.

We next ask whether the same asymmetry also appears in learned prediction. We train an FNO on lower frequencies and evaluate it at unseen higher frequencies. 
Fig.~\ref{fig:em_similarity}(c) shows relative amplitude and phase similarity across SimpleWave, Helmholtz, and Maxwell benchmarks, obtained by normalizing each higher-frequency prediction similarity by that at the highest observed lower frequency.
Benchmark details are provided in Appendix~\ref{app:dataset_details}. 
In all three cases, amplitude similarity remains substantially higher than phase similarity as the target frequency moves away from the observed regime. 
Thus, direct lower-to-higher extrapolation does not fail uniformly across the complex field: the phase-sensitive component breaks down earlier and more severely than the amplitude component.

The empirical evidence above suggests that the central difficulty of higher-frequency prediction is not merely larger error, but uneven degradation across different components of the complex field. To formalize this asymmetry, we adopt a local amplitude--phase abstraction. Let $u$ and $\hat u$ denote the true and predicted fields, respectively, and write
\begin{equation}
	u(\mathbf r,\nu)=A(\mathbf r;\nu)\exp\bigl(i\phi(\mathbf r;\nu)\bigr),
	\qquad
	\hat u(\mathbf r,\nu)=\hat A(\mathbf r;\nu)\exp\bigl(i\hat\phi(\mathbf r;\nu)\bigr).
\end{equation}
In the local abstraction below, we express phase through effective travel-time functions, i.e., $\phi(\mathbf r;\nu)\approx \nu\tau(\mathbf r)$ and $\hat\phi(\mathbf r;\nu)\approx \nu\hat\tau(\mathbf r)$, and define
\begin{equation}
	\Delta\tau(\mathbf r):=\hat\tau(\mathbf r)-\tau(\mathbf r).
\end{equation}
Since our interest is in whole-field discrepancy rather than a single spatial point, we measure the error over a measurable region $S\subseteq\Omega$ by
\begin{equation}
	\|\hat u-u\|_{L^2(S)}^2
	=
	\int_S |\hat u(\mathbf r,\nu)-u(\mathbf r,\nu)|^2\,d\mathbf r.
\end{equation}
This regional quantity is more natural than a pointwise error here because both the empirical observations in Fig.~\ref{fig:em_similarity} and the prediction task itself concern the field as a whole. The result below follows by integrating the corresponding pointwise amplitude--phase decomposition over $S$; the pointwise form and proof are given in Appendix~\ref{app:proof_freq_amp_phase}.

\begin{proposition}[Regional decomposition of complex-field error]
\label{prop:regional_exact_decomp}
For any measurable region $S\subseteq\Omega$,
\begin{equation}
	\|\hat u-u\|_{L^2(S)}^2
	=
	\|\hat A-A\|_{L^2(S)}^2
	+
	4\int_S A(\mathbf r;\nu)\hat A(\mathbf r;\nu)\,
	\sin^2\!\left(\frac{\nu\Delta\tau(\mathbf r)}{2}\right)\,d\mathbf r.
	\label{eq:regional_exact_decomp}
\end{equation}
\end{proposition}

Proposition~\ref{prop:regional_exact_decomp} makes the asymmetry explicit at the level of the regional field error. Amplitude mismatch contributes through a direct term, whereas phase-sensitive mismatch contributes through an oscillatory term whose argument is scaled by frequency. The two components therefore enter higher-frequency prediction error through qualitatively different mechanisms.

The following corollary further highlights the frequency dependence through a simple upper bound.

\begin{corollary}[Frequency-sensitive upper bound for the regional field error]
\label{cor:regional_upper_bound}
If $A(\mathbf r;\nu)\le A_{\max}$ and $\hat A(\mathbf r;\nu)\le \hat A_{\max}$ on $S$, then
\begin{equation}
	\|\hat u-u\|_{L^2(S)}^2
	\le
	\|\hat A-A\|_{L^2(S)}^2
	+
	\nu^2 A_{\max}\hat A_{\max}\,
	\|\Delta\tau\|_{L^2(S)}^2.
	\label{eq:regional_uniform_upper_bound}
\end{equation}
\end{corollary}

Corollary~\ref{cor:regional_upper_bound} is explanatory rather than tight: even with controlled amplitude mismatch,
phase-sensitive propagation mismatch carries an explicit frequency-dependent factor and can grow
rapidly in the higher-frequency regime.


\begin{tcolorbox}[
	enhanced,
	colback=gray!10,
	colframe=black,
	boxrule=0.7pt,
	arc=7pt,
	left=10pt,right=10pt,top=7pt,bottom=7pt,
	drop shadow={black!12!white}
	]
	\textbf{Design Guideline.}
	Cross-frequency transfer in complex wave fields is inherently asymmetric: coarse amplitude remains comparatively transferable across frequencies, whereas phase-sensitive oscillatory detail is substantially more fragile outside the observed regime. Accordingly, higher-frequency prediction should reuse transferable coarse structure and separately recover the missing oscillatory detail.
\end{tcolorbox}
\section{Method}

\begin{figure*}[t]
	\centering
	\includegraphics[width=\textwidth]{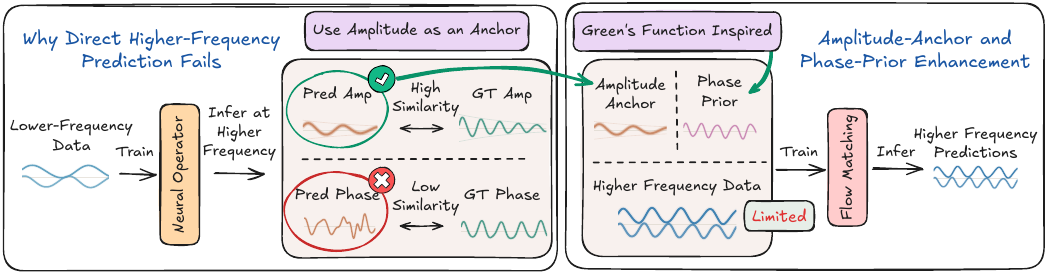}
	\caption{Overview of the proposed framework. Left: direct higher-frequency prediction by a lower-frequency neural operator preserves amplitude structure more reliably than phase. 
    Right: the extrapolated coarse amplitude is retained as an anchor and combined with a Green's-function-inspired phase prior to guide conditional flow matching for higher-frequency enhancement.}
	\label{fig:pipeline}
\end{figure*}

\subsection{Problem Setup}

We consider higher-frequency prediction under scarce target supervision for parameterized wave fields. Let $e\in\mathcal E$ denote an environment, $\nu\in\mathcal V$ denote the spectral variable, and $u(\mathbf r;e,\nu)\in\mathbb C^{H\times W}$ denote the target complex-valued field on a spatial grid $\mathbf r$. We partition the spectrum into a lower-frequency source regime $\mathcal V_{\mathrm{LF}}$ and a disjoint higher-frequency target regime $\mathcal V_{\mathrm{HF}}$, and aim to predict $u(\mathbf r;e,\nu)$ for $\nu\in\mathcal V_{\mathrm{HF}}$ using supervision concentrated primarily in $\mathcal V_{\mathrm{LF}}$ together with limited target-frequency samples.

Motivated by the transfer asymmetry established in Sec.~\ref{sec:asymmetry_transferability}, our design separates higher-frequency prediction into two parts: reusing the comparatively transferable coarse structure from a frozen lower-frequency backbone, and recovering the missing oscillatory detail using a conditional enhancement model. The overall pipeline is illustrated in Fig.~\ref{fig:pipeline}.

\subsection{Coarse Amplitude Anchor}
\label{sec:3.2}
We first extract the transferable component of lower-to-higher prediction. 
A neural operator $F_\theta$ is trained on the lower-frequency regime $\mathcal V_{\mathrm{LF}}$, frozen after training, and then evaluated at target frequencies $\nu\in\mathcal V_{\mathrm{HF}}$. 
Its output is treated as a coarse intermediate prediction rather than as the final higher-frequency solution. We define the corresponding coarse amplitude anchor by
\begin{equation}
	a_{\mathrm{coarse}}(\mathbf r;e,\nu):=
	\log\left(\big|F_\theta(e,\nu)(\mathbf r)\big|+\epsilon\right)
	\in\mathbb R^{H\times W}.
	\label{eq:coarse_anchor}
\end{equation}
Here, \(\epsilon\) is a small positive constant for numerical stability. We retain only the log-amplitude of the extrapolated prediction as a coarse target-regime structural reference, leaving phase-sensitive oscillations to subsequent refinement. 
In this work, \(F_\theta\) is instantiated with FNO for simplicity, but the amplitude-anchor construction is backbone-agnostic in principle and can be built on any surrogate that provides a coarse target-frequency prediction.

\subsection{Green's-Function-Inspired Phase Prior}
\label{sec:3.3}

The coarse amplitude anchor captures transferable large-scale structure, but it does not by itself provide the oscillatory phase information required in the higher-frequency regime. We therefore introduce an explicit phase prior that supplies a lightweight physical guide for the missing phase-sensitive detail. Its role is not to reproduce the full target field, but to provide a compact oscillatory scaffold consistent with the source, the geometry, and an effective phase scale.

Our construction is based on a simple geometric-propagation view of phase: when a small number of dominant propagation paths govern the main oscillatory behavior, the accumulated phase along each path can be approximated by an effective phase coefficient multiplied by the corresponding propagation distance. Based on this intuition, we use the analytic surrogate
\begin{equation}
	G_{\mathrm{prior}}(\mathbf r,\mathbf r_s;e,\nu)
	=
	\sum_{m=1}^{M}
	a_m
	\exp\bigl(
	i\,\kappa_{\mathrm{ref}}(e,\nu)\,L_m(\mathbf r,\mathbf r_s)
	\bigr),
	\label{eq:green_prior}
\end{equation}
where $\mathbf r_s$ denotes the source location, $M$ is the number of retained dominant propagation paths, $L_m(\mathbf r,\mathbf r_s)$ is the geometric path length from the source to location $\mathbf r$, $\kappa_{\mathrm{ref}}(e,\nu)$ is a domain-adaptive reference phase coefficient, and $a_m$ is a fixed coefficient associated with the $m$-th retained path. 
Eq.~\eqref{eq:green_prior} is a unified low-complexity surrogate for phase accumulation, rather than a full approximation of the Green's function or the target field. The retained paths are therefore chosen to capture the dominant propagation geometry rather than to enumerate all physical paths.
Concrete instantiations for the experimental domains are provided in Appendix~\ref{app:phase_prior_inst}.

Since transferable amplitude is already supplied by Sec.~\ref{sec:3.2}, we use only the surrogate phase:
\begin{equation}
	\phi_{\mathrm{base}}(\mathbf r;e,\nu)
	:=
	\operatorname{Arg}\bigl(
	G_{\mathrm{prior}}(\mathbf r,\mathbf r_s;e,\nu)
	\bigr).
\end{equation}
Here, $\operatorname{Arg}(\cdot)$ denotes the principal argument, with values in $(-\pi,\pi]$. This phase map provides source-aware and geometry-aware oscillatory cues, while the remaining higher-frequency detail is recovered by the enhancement model.

\subsection{Amplitude- and Phase-Prior-Guided Conditional Flow Matching}


We use the coarse amplitude anchor and the phase prior as complementary conditions for a generative enhancement model: the former provides transferable large-scale structure, while the latter supplies source- and geometry-aligned oscillatory cues.

Let $\phi=\operatorname{Arg}(u)$ denote the target phase. We represent the target field by
\begin{equation}
	x_1=
	\bigl[\log (|u|+\epsilon),\; \sin\phi,\; \cos\phi\bigr],
\end{equation}
where the logarithm compresses the dynamic range of the amplitude and the sine--cosine representation avoids the discontinuity of wrapped phase at $\pm\pi$. The conditioning variables are
\begin{equation}
	c=
	\bigl(
	a_{\mathrm{coarse}},\,
	[\sin\phi_{\mathrm{base}},\cos\phi_{\mathrm{base}}],\,
	z_f
	\bigr),
\end{equation}
where $a_{\mathrm{coarse}}$ is the coarse amplitude anchor from Sec.~\ref{sec:3.2}, $[\sin\phi_{\mathrm{base}},\cos\phi_{\mathrm{base}}]$ are the phase-prior features from Sec.~\ref{sec:3.3}, and $z_f$ is a Fourier feature embedding of the spectral variable $\nu$.

The conditional velocity field is parameterized by a U-Net~\cite{ronneberger2015u}. The coarse amplitude and phase-prior maps are injected through multi-scale spatial conditioning so that both global field organization and local oscillatory cues remain available across resolution levels, while $z_f$ modulates the network through FiLM-style conditioning~\cite{perez2018film} to adapt the feature response to different target frequencies.

Conditional flow matching (CFM) is used to model the enhancement process~\cite{lipman2023flow,tong2024improving}. With $x_0\sim\mathcal N(0,I)$ and the linear interpolation
\begin{equation}
	x_t=(1-t)x_0+t x_1,
\end{equation}
the network $v_\psi(x_t,t;c)$ is trained with
\begin{equation}
	\mathcal L_{\mathrm{FM}}
	=
	\mathbb E_{x_0,x_1,t}
	\left[
	\|v_\psi(x_t,t;c)-(x_1-x_0)\|_2^2
	\right].
\end{equation}
At inference, we follow the standard CFM sampling procedure conditioned on \(c\).
\section{Related Work}

\textbf{Neural operators} have become a standard framework for learning parametric PDE solution maps in function space, with representative formulations including DeepONet and FNO~\cite{lu2021learning,lifourier,kovachki2023neural}. Subsequent variants improve these models through transformer-based, geometry-aware, and multiscale designs~\cite{cao2021choose,hao2023gnot,li2023geometry,rahman2023uno,cao2024laplace}. 
Recent work has also begun to study cross-frequency prediction in wave problems. \cite{gu2022neurolight,seo2024wave} consider cross-wavelength prediction within an observed spectral regime. \cite{sun2025hankel,wang2024transfer} show that cross-frequency generalization can be improved through adaptation or fine-tuning on the target regime. \citep{benitez2024out,song2022high} further examine out-of-distribution behavior and supervised higher-frequency wavefield extrapolation. Compared with this line of work, our study focuses more directly on what fails to transfer across frequencies in oscillatory wave fields 
and how that failure should be handled.

\textbf{Generative surrogate models }\citep{ho2020denoising,song2021score} have recently become a useful tool for scientific field modeling and PDE-related inference. 
Beyond standard diffusion formulations, dynamics-aware generative models have been used for long-horizon PDE refinement, and spatiotemporal physical prediction \citep{lippe2023pde,cachay2023dyffusion}. 
Recent studies further extend generative modeling to physics-aware field reconstruction and PDE-governed scientific simulation \citep{bastek2025physics,huang2024diffusionpde,li2025physics,chen2025generating,zou2026probabilistic}. 
In contrast, our work targets scarce-supervision higher-frequency wave prediction, where the generative model is used as an enhancer guided by transferable amplitude structure and a phase-sensitive oscillatory prior.

\section{Experiments}
\begin{figure}[!t]
	\centering
	\includegraphics[width=\textwidth]{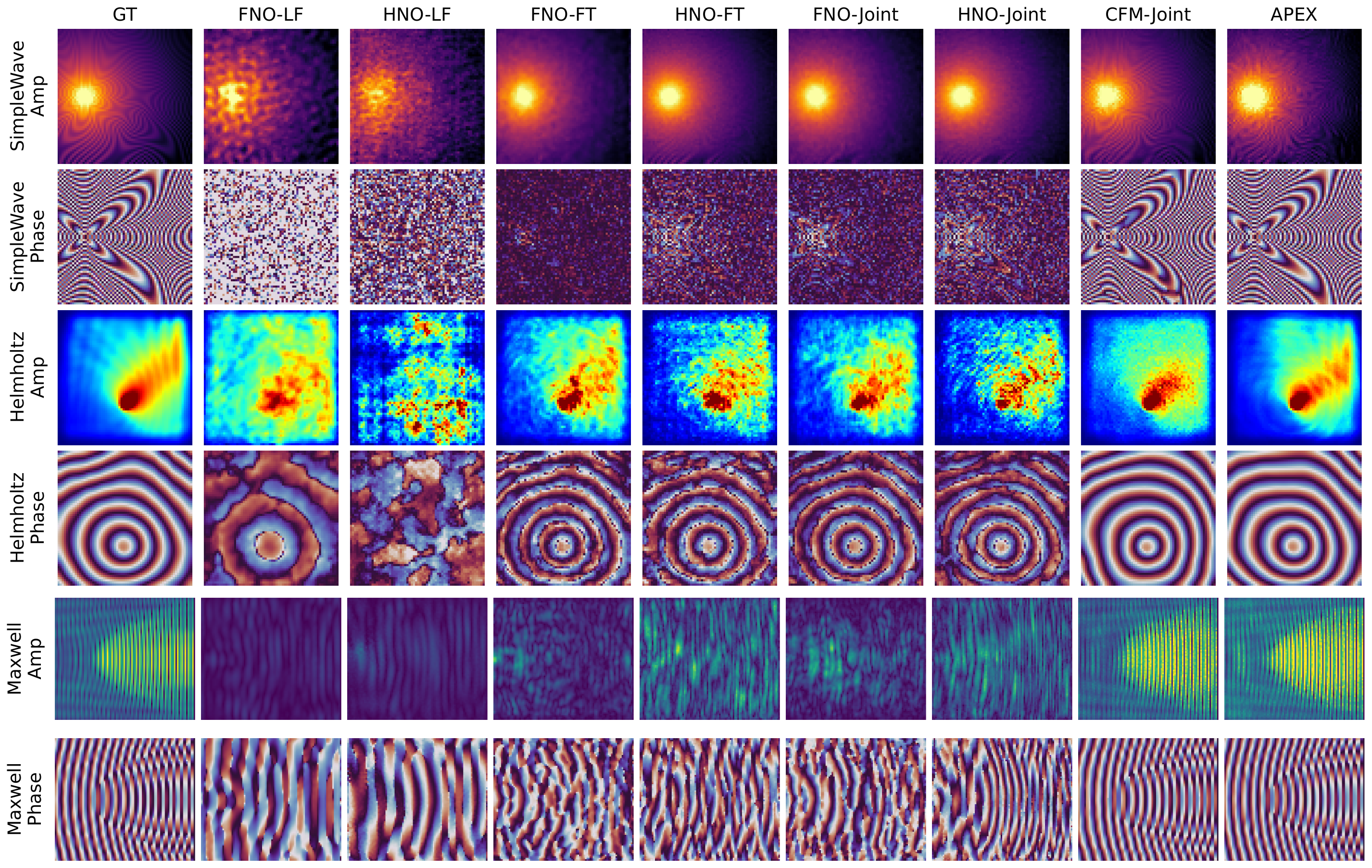}
	\caption{Qualitative comparison at HF100, the largest frequency gap considered in this paper. Across three benchmarks, APEX recovers higher-frequency amplitude and phase patterns with fewer structural distortions and better alignment with the ground truth.}
	\label{fig:hf100_qualitative}
\end{figure}

\begin{table}[!t]
	\centering
	\small
	\setlength{\tabcolsep}{4pt}
	\renewcommand{\arraystretch}{1.08}
	\caption{Higher-frequency prediction results across three benchmarks. Each frequency group reports $H^1$/AWPC. Underlined entries denote the best result within each dataset and metric column.}
	\label{tab:main_results}
	\begin{tabular}{lcccccccc}
		\toprule
		\multicolumn{9}{c}{\textbf{SimpleWave}} \\
		\midrule
		\textbf{Method} & \multicolumn{2}{c}{\textbf{HF20}} & \multicolumn{2}{c}{\textbf{HF50}} & \multicolumn{2}{c}{\textbf{HF100}} & \multicolumn{2}{c}{\textbf{Overall}} \\
		\cmidrule(lr){2-3}\cmidrule(lr){4-5}\cmidrule(lr){6-7}\cmidrule(lr){8-9}
		& \textbf{$H^1 \downarrow$} & \textbf{AWPC$\uparrow$} & \textbf{$H^1 \downarrow$} & \textbf{AWPC$\uparrow$} & \textbf{$H^1 \downarrow$} & \textbf{AWPC$\uparrow$} & \textbf{$H^1 \downarrow$} & \textbf{AWPC$\uparrow$} \\
		\midrule
		FNO-LF & 1.377 & 0.012 & 1.454 & 0.013 & 1.429 & 0.013 & 1.420 & 0.013 \\
		FNO-FT & 1.053 & 0.157 & 1.114 & 0.238 & 1.128 & 0.133 & 1.099 & 0.176 \\
		HNO-LF & 1.359 & 0.012 & 1.420 & 0.014 & 1.401 & 0.014 & 1.393 & 0.013 \\
		HNO-FT & 0.923 & 0.496 & 1.021 & 0.465 & 1.071 & 0.400 & 1.005 & 0.454 \\
		FNO-Joint & 0.963 & 0.403 & 1.032 & 0.438 & 1.100 & 0.324 & 1.032 & 0.388 \\
		HNO-Joint & 0.885 & 0.540 & 0.973 & 0.531 & 1.044 & 0.433 & 0.967 & 0.501 \\
		CFM-Joint & 0.935 & 0.620 & 1.151 & 0.516 & 1.234 & 0.351 & 1.107 & 0.496 \\
        APEX & \underline{0.396} & \underline{0.909} & \underline{0.494} & \underline{0.865} & \underline{0.804} & \underline{0.659} & \underline{0.564} & \underline{0.811} \\
		\midrule
		\multicolumn{9}{c}{\textbf{Helmholtz}} \\
		\midrule
		\textbf{Method} & \multicolumn{2}{c}{\textbf{HF20}} & \multicolumn{2}{c}{\textbf{HF50}} & \multicolumn{2}{c}{\textbf{HF100}} & \multicolumn{2}{c}{\textbf{Overall}} \\
		\cmidrule(lr){2-3}\cmidrule(lr){4-5}\cmidrule(lr){6-7}\cmidrule(lr){8-9}
		& \textbf{$H^1 \downarrow$} & \textbf{AWPC$\uparrow$} & \textbf{$H^1 \downarrow$} & \textbf{AWPC$\uparrow$} & \textbf{$H^1 \downarrow$} & \textbf{AWPC$\uparrow$} & \textbf{$H^1 \downarrow$} & \textbf{AWPC$\uparrow$} \\
		\midrule
		FNO-LF & 1.611 & 0.169 & 1.686 & 0.098 & 2.261 & 0.050 & 1.853 & 0.106 \\
		FNO-FT & 0.762 & 0.897 & 0.783 & 0.874 & 0.966 & 0.759 & 0.837 & 0.843 \\
		HNO-LF & 2.989 & 0.097 & 3.743 & 0.042 & 5.637 & 0.032 & 4.123 & 0.057 \\
		HNO-FT & 0.756 & 0.871 & 0.909 & 0.821 & 1.276 & 0.696 & 0.980 & 0.796 \\
		FNO-Joint & 0.715 & 0.903 & 0.779 & 0.850 & 0.899 & 0.737 & 0.797 & 0.830 \\
		HNO-Joint & 0.930 & 0.853 & 1.006 & 0.794 & 1.401 & 0.662 & 1.112 & 0.770 \\
		CFM-Joint & 0.852 & 0.653 & 0.895 & 0.543 & 0.930 & 0.397 & 0.893 & 0.531 \\
		APEX & \underline{0.422} & \underline{0.928} & \underline{0.492} & \underline{0.885} & \underline{0.634} & \underline{0.815} & \underline{0.516} & \underline{0.876} \\
		\midrule
		\multicolumn{9}{c}{\textbf{Maxwell}} \\
		\midrule
		\textbf{Method} & \multicolumn{2}{c}{\textbf{HF20}} & \multicolumn{2}{c}{\textbf{HF50}} & \multicolumn{2}{c}{\textbf{HF100}} & \multicolumn{2}{c}{\textbf{Overall}} \\
		\cmidrule(lr){2-3}\cmidrule(lr){4-5}\cmidrule(lr){6-7}\cmidrule(lr){8-9}
		& \textbf{$H^1 \downarrow$} & \textbf{AWPC$\uparrow$} & \textbf{$H^1 \downarrow$} & \textbf{AWPC$\uparrow$} & \textbf{$H^1 \downarrow$} & \textbf{AWPC$\uparrow$} & \textbf{$H^1 \downarrow$} & \textbf{AWPC$\uparrow$} \\
		\midrule
		FNO-LF & 1.090 & 0.066 & 1.023 & 0.003 & 1.008 & 0.000 & 1.024 & 0.017 \\
		FNO-FT & 0.840 & 0.465 & 0.899 & 0.301 & 0.954 & 0.169 & 0.925 & 0.284 \\
		HNO-LF & 1.097 & 0.044 & 1.045 & 0.012 & 1.014 & 0.002 & 1.034 & 0.007 \\
		HNO-FT & 0.912 & 0.421 & 0.947 & 0.329 & 0.962 & 0.277 & 0.951 & 0.329 \\
		FNO-Joint & 0.904 & 0.422 & 0.933 & 0.337 & 0.959 & 0.233 & 0.945 & 0.313 \\
		HNO-Joint & 0.842 & 0.441 & 0.900 & 0.342 & 0.948 & 0.285 & 0.922 & 0.342 \\
		CFM-Joint & 0.701 & \underline{0.828} & 0.897 & 0.680 & 1.187 & 0.411 & 1.061 & 0.595 \\
		APEX & \underline{0.697} & 0.812 & \underline{0.790} & \underline{0.725} & \underline{0.929} & \underline{0.624} & \underline{0.865} & \underline{0.703} \\
		\bottomrule
	\end{tabular}
\end{table}

\paragraph{Experimental setup.}

The evaluation is conducted on three wave benchmarks: SimpleWave, Helmholtz, and Maxwell. SimpleWave and Helmholtz use scalar complex fields on fixed $64\times64$ grids, while Maxwell uses the dominant polarization component $E_y$ on an $80\times92$ grid. All benchmarks use the same frequency-split protocol: four lower-frequency (LF) values and three \textbf{higher-frequency (HF)} values, with $N$ samples per frequency. The $4N$ LF samples are split $80\%/20\%$ for training/testing, and the $3N$ HF samples are split $20\%/80\%$. Evaluation is performed only on the HF test split; the LF test split is used only for lower-frequency monitoring and is not included in the reported HF metrics.
\textbf{HF20}, \textbf{HF50}, and \textbf{HF100} denote target groups whose frequencies exceed the maximum LF training frequency by $20\%$, $50\%$, and $100\%$, respectively.
Additional details are in Appendix~\ref{app:dataset_details}. The code and datasets are available at \url{https://anonymous.4open.science/r/apex-B88A/}.

The comparison includes three baseline model families: the Fourier neural operator (FNO) \citep{lifourier}, the Helmholtz neural operator (HNO) \citep{zou2024deep}, and conditional flow matching (CFM) \citep{lipman2023flow,tong2024improving}. For the operator baselines, we consider three training protocols. `LF' denotes LF-only training followed by direct extrapolation to the target frequencies. `FT' initializes from the LF model and then fine-tunes on the available HF training data. `Joint' denotes joint training on all available LF and HF training samples.
This yields FNO-LF/FT/Joint and HNO-LF/FT/Joint as operator baselines, while CFM-Joint is the generative baseline trained jointly on available LF and HF data. APEX first trains and freezes the LF operator on LF samples, then trains the enhancer using the same HF training split.

Performance is evaluated with two complex-field metrics. The relative complex $H^1$ error is \cite{qiu2024derivative}
\begin{equation}
H^1(\hat{u},u)=
\left(
\frac{
	\|\hat{u}-u\|_2^2
	+
	\|\Delta_x(\hat{u}-u)\|_2^2
	+
	\|\Delta_y(\hat{u}-u)\|_2^2
}{
	\|u\|_2^2
	+
	\|\Delta_x u\|_2^2
	+
	\|\Delta_y u\|_2^2
}
\right)^{1/2},
\label{eq:h1_metric}
\end{equation}
where $\hat{u}$ and $u$ denote the predicted and ground-truth complex wave fields, and $\Delta_x,\Delta_y$ are discrete forward differences along the spatial directions. Amplitude-weighted phase coherence (AWPC) is
\begin{equation}
\mathrm{AWPC}(\hat{u},u)=
\left|
\frac{
	\sum_{\mathbf{x}} |u(\mathbf{x})|
	\exp\!\big(i(\operatorname{Arg} (\hat{u}(\mathbf{x}))-\operatorname{Arg} (u(\mathbf{x})))\big)
}{
	\sum_{\mathbf{x}} |u(\mathbf{x})|
}
\right|,
\label{eq:awpc_metric}
\end{equation}
which evaluates phase agreement while weighting each spatial location by the magnitude of the ground-truth field \cite{barnes2024phase,cardenas2026exercise}. All reported results are averaged over the corresponding test set. Additional uncertainty estimates for the main results are reported in Appendix~\ref{app:bootstrap_main_results}.

\paragraph{Main results.}
The proposed method achieves the strongest overall higher-frequency prediction performance across the three benchmarks. As shown in Fig.~\ref{fig:hf100_qualitative}, direct LF extrapolation retains some coarse amplitude structure but exhibits clear phase and oscillatory mismatch. Fine-tuned and jointly trained operator baselines improve the coarse pattern but remain limited under scarce target-frequency supervision, while `CFM-Joint' can produce plausible structures that are not always aligned with the target field. Additional qualitative comparisons at HF20 and HF50 are provided in Appendix~\ref{app:additional_qualitative}. Tab.~\ref{tab:main_results} further confirms this trend: APEX achieves the best overall results on all three benchmarks and leads in nearly all metric columns, with only a small AWPC gap at Maxwell HF20. These results support the combination of transferable coarse amplitude and physics-guided phase priors for higher-frequency wave-field prediction.

\paragraph{Ablation study.}
The ablation results in Tab.~\ref{tab:main_ablation} examine the two main design components. Removing the phase prior consistently degrades both $H^1$ and AWPC, suggesting that the coarse amplitude anchor alone does not fully recover the required oscillatory structure. Removing the coarse anchor leads to a larger drop, consistent with its role as the main transferable structural reference. The two components therefore appear complementary: the coarse prediction supports global organization, while the phase prior improves fine oscillatory recovery.

\begin{table}[!t]
\centering
\small
\setlength{\tabcolsep}{4pt}
\renewcommand{\arraystretch}{1.08}
\caption{Ablation results on the Helmholtz benchmark.}
\label{tab:main_ablation}
\begin{tabular}{lcccccccc}
\toprule
\textbf{Method} & \multicolumn{2}{c}{\textbf{HF20}} & \multicolumn{2}{c}{\textbf{HF50}} & \multicolumn{2}{c}{\textbf{HF100}} & \multicolumn{2}{c}{\textbf{Overall}} \\
\cmidrule(lr){2-3}\cmidrule(lr){4-5}\cmidrule(lr){6-7}\cmidrule(lr){8-9}
& \textbf{$H^1 \downarrow$} & \textbf{AWPC$\uparrow$} & \textbf{$H^1 \downarrow$} & \textbf{AWPC$\uparrow$} & \textbf{$H^1 \downarrow$} & \textbf{AWPC$\uparrow$} & \textbf{$H^1 \downarrow$} & \textbf{AWPC$\uparrow$} \\
\midrule
w/o Phase Prior & 0.483 & 0.880 & 0.580 & 0.809 & 0.812 & 0.689 & 0.625 & 0.793 \\
w/o Coarse Anchor & 0.907 & 0.628 & 1.081 & 0.505 & 1.252 & 0.379 & 1.080 & 0.504 \\
APEX & \underline{0.422} & \underline{0.928} & \underline{0.492} & \underline{0.885} & \underline{0.634} & \underline{0.815} & \underline{0.516} & \underline{0.876} \\
\bottomrule
\end{tabular}
\end{table}

\paragraph{Sensitivity to target-frequency supervision.}
Tab.~\ref{tab:ratio_helmholtz} fixes the lower-frequency split and varies only the training/testing ratio in the higher-frequency regime, isolating the effect of target-frequency supervision. 
Across different higher-frequency supervision ratios, APEX remains consistently better than `CFM-Joint', indicating that its advantage is not tied to a particular amount of target-frequency data.
One possible reason is that, in joint generative training, the lower-frequency samples still dominate the training distribution, making it difficult for `CFM-Joint' to sufficiently adapt to the scarce higher-frequency regime. This is consistent with Fig.~\ref{fig:hf100_qualitative}: `CFM-Joint' can generate visually plausible structures, but these structures are not always aligned with the target field. Therefore, the main limitation is not only the amount of target-frequency data, but also the lack of explicit constraints that align the generated field with the target higher-frequency response. APEX addresses this by using transferable structure and phase prior to guide the generated field toward the target response.

\begin{table}[!t]
\centering
\scriptsize
\setlength{\tabcolsep}{3.5pt}
\renewcommand{\arraystretch}{1.06}
\caption{Sensitivity to target-frequency supervision on the Helmholtz benchmark. Underlines mark the better result between CFM-Joint and APEX at each ratio.}
\label{tab:ratio_helmholtz}
\begin{tabular}{llcccccccc}
\toprule
\textbf{Ratio} & \textbf{Method} & \multicolumn{2}{c}{\textbf{HF20}} & \multicolumn{2}{c}{\textbf{HF50}} & \multicolumn{2}{c}{\textbf{HF100}} & \multicolumn{2}{c}{\textbf{Overall}} \\
\cmidrule(lr){3-4}\cmidrule(lr){5-6}\cmidrule(lr){7-8}\cmidrule(lr){9-10}
& & \textbf{$H^1 \downarrow$} & \textbf{AWPC$\uparrow$} & \textbf{$H^1 \downarrow$} & \textbf{AWPC$\uparrow$} & \textbf{$H^1 \downarrow$} & \textbf{AWPC$\uparrow$} & \textbf{$H^1 \downarrow$} & \textbf{AWPC$\uparrow$} \\
\midrule
1/9 & CFM-Joint & 0.859 & 0.655 & 0.912 & 0.489 & 0.936 & 0.409 & 0.902 & 0.518 \\
1/9 & APEX & \underline{0.656} & \underline{0.819} & \underline{0.721} & \underline{0.760} & \underline{0.906} & \underline{0.635} & \underline{0.761} & \underline{0.738} \\
\midrule
2/8 & CFM-Joint & 0.852 & 0.653 & 0.895 & 0.543 & 0.930 & 0.397 & 0.893 & 0.531 \\
2/8 & APEX & \underline{0.422} & \underline{0.928} & \underline{0.492} & \underline{0.885} & \underline{0.634} & \underline{0.815} & \underline{0.516} & \underline{0.876} \\
\midrule
3/7 & CFM-Joint & 0.878 & 0.603 & 0.908 & 0.521 & 0.938 & 0.401 & 0.908 & 0.508 \\
3/7 & APEX & \underline{0.494} & \underline{0.929} & \underline{0.534} & \underline{0.902} & \underline{0.663} & \underline{0.829} & \underline{0.564} & \underline{0.886} \\
\midrule
4/6 & CFM-Joint & 0.865 & 0.661 & 0.905 & 0.540 & 0.933 & 0.418 & 0.901 & 0.539 \\
4/6 & APEX & \underline{0.372} & \underline{0.960} & \underline{0.441} & \underline{0.933} & \underline{0.626} & \underline{0.893} & \underline{0.480} & \underline{0.929} \\
\bottomrule
\end{tabular}
\end{table}

\section{Conclusion}
This paper studies higher-frequency prediction of complex wave fields under limited target-frequency supervision. We identify an asymmetric transfer pattern: coarse amplitude structure remains relatively stable, whereas phase-sensitive oscillatory structure deteriorates rapidly under lower-to-higher extrapolation. APEX leverages this asymmetry by using a frozen lower-frequency neural operator as a coarse amplitude anchor and a Green's-function-inspired phase prior to guide conditional flow matching for oscillatory recovery. Experiments, transferability analysis, and ablations across three wave benchmarks support this decomposition and show stronger performance than direct lower-to-higher extrapolation, target-adapted operator, and joint generative baselines. A limitation is that the phase prior is simplified; future work will incorporate stronger physics-aware priors and constraints for more accurate and physically consistent higher-frequency prediction.
\bibliographystyle{plainnat}
\bibliography{refs}

\clearpage
\appendix

\section{Proofs for the Amplitude--Phase Error Decomposition}
\label{app:proof_freq_amp_phase}

This appendix provides the proofs of the analytical results stated in Sec.~\ref{sec:asymmetry_transferability}. We first derive an exact pointwise amplitude--phase error decomposition, then integrate it over a spatial region to obtain Proposition~\ref{prop:regional_exact_decomp}, and finally establish Corollary~\ref{cor:regional_upper_bound}.

\subsection{Exact Pointwise Decomposition}

\begin{lemma}[Exact pointwise amplitude--phase error decomposition]
\label{lem:pointwise_exact_decomp}
Fix a target frequency $\nu>0$. Let
\begin{equation}
u(\mathbf r,\nu)=A(\mathbf r;\nu)\exp\bigl(i\nu\tau(\mathbf r)\bigr),
\qquad
\hat u(\mathbf r,\nu)=\hat A(\mathbf r;\nu)\exp\bigl(i\nu\hat\tau(\mathbf r)\bigr),
\end{equation}
and define
\begin{equation}
\Delta\tau(\mathbf r):=\hat\tau(\mathbf r)-\tau(\mathbf r).
\end{equation}
Then, for every $\mathbf r\in\Omega$,
\begin{equation}
|\hat u(\mathbf r,\nu)-u(\mathbf r,\nu)|^2
=
\bigl(\hat A(\mathbf r;\nu)-A(\mathbf r;\nu)\bigr)^2
+
4A(\mathbf r;\nu)\hat A(\mathbf r;\nu)
\sin^2\!\left(\frac{\nu\Delta\tau(\mathbf r)}{2}\right).
\label{eq:pointwise_exact_decomp}
\end{equation}
\end{lemma}

\begin{proof}
Fix $\mathbf r\in\Omega$ and suppress the dependence on $(\mathbf r,\nu)$ when no ambiguity arises. By direct expansion,
\begin{align}
|\hat u-u|^2
&=
\left|
\hat A e^{i\nu\hat\tau}-A e^{i\nu\tau}
\right|^2 \\
&=
\hat A^2 + A^2 - 2A\hat A \cos\bigl(\nu(\hat\tau-\tau)\bigr).
\end{align}
Using $\Delta\tau=\hat\tau-\tau$ and the identity
\begin{equation}
\cos x = 1-2\sin^2(x/2),
\end{equation}
we obtain
\begin{align}
|\hat u-u|^2
&=
\hat A^2 + A^2 - 2A\hat A\Bigl(1-2\sin^2(\nu\Delta\tau/2)\Bigr) \\
&=
(\hat A-A)^2 + 4A\hat A\sin^2\!\left(\frac{\nu\Delta\tau}{2}\right),
\end{align}
which proves \eqref{eq:pointwise_exact_decomp}.
\end{proof}

\subsection{Proof of Proposition~\ref{prop:regional_exact_decomp}}

\begin{proof}[Proof of Proposition~\ref{prop:regional_exact_decomp}]
By definition of the regional $L^2$ field error,
\begin{equation}
\|\hat u-u\|_{L^2(S)}^2
=
\int_S |\hat u(\mathbf r,\nu)-u(\mathbf r,\nu)|^2\,d\mathbf r.
\end{equation}
Applying Lemma~\ref{lem:pointwise_exact_decomp} pointwise inside the integral yields
\begin{align}
\|\hat u-u\|_{L^2(S)}^2
&=
\int_S
\left[
\bigl(\hat A(\mathbf r;\nu)-A(\mathbf r;\nu)\bigr)^2
+
4A(\mathbf r;\nu)\hat A(\mathbf r;\nu)
\sin^2\!\left(\frac{\nu\Delta\tau(\mathbf r)}{2}\right)
\right]d\mathbf r \\
&=
\|\hat A-A\|_{L^2(S)}^2
+
4\int_S A(\mathbf r;\nu)\hat A(\mathbf r;\nu)
\sin^2\!\left(\frac{\nu\Delta\tau(\mathbf r)}{2}\right)\,d\mathbf r,
\end{align}
which is exactly \eqref{eq:regional_exact_decomp}.
\end{proof}

\subsection{Proof of Corollary~\ref{cor:regional_upper_bound}}

\begin{proof}[Proof of Corollary~\ref{cor:regional_upper_bound}]
Starting from Proposition~\ref{prop:regional_exact_decomp},
\begin{equation}
\|\hat u-u\|_{L^2(S)}^2
=
\|\hat A-A\|_{L^2(S)}^2
+
4\int_S A(\mathbf r;\nu)\hat A(\mathbf r;\nu)
\sin^2\!\left(\frac{\nu\Delta\tau(\mathbf r)}{2}\right)\,d\mathbf r.
\end{equation}
Using the elementary inequality
\begin{equation}
\sin z \le z, \qquad z\ge 0,
\end{equation}
we obtain
\begin{equation}
\sin^2\!\left(\frac{\nu\Delta\tau(\mathbf r)}{2}\right)
=
\sin^2\!\left(\frac{\nu|\Delta\tau(\mathbf r)|}{2}\right)
\le
\left(\frac{\nu|\Delta\tau(\mathbf r)|}{2}\right)^2
=
\frac{\nu^2}{4}\Delta\tau(\mathbf r)^2.
\end{equation}
Substituting this bound into Proposition~\ref{prop:regional_exact_decomp} gives
\begin{align}
\|\hat u-u\|_{L^2(S)}^2
&\le
\|\hat A-A\|_{L^2(S)}^2
+
4\int_S A(\mathbf r;\nu)\hat A(\mathbf r;\nu)\cdot
\frac{\nu^2}{4}\Delta\tau(\mathbf r)^2\,d\mathbf r \\
&=
\|\hat A-A\|_{L^2(S)}^2
+
\nu^2\int_S A(\mathbf r;\nu)\hat A(\mathbf r;\nu)\Delta\tau(\mathbf r)^2\,d\mathbf r,
\end{align}
If, in addition, $A(\mathbf r;\nu)\le A_{\max}$ and $\hat A(\mathbf r;\nu)\le \hat A_{\max}$ on $S$, then
\begin{equation}
A(\mathbf r;\nu)\hat A(\mathbf r;\nu)\le A_{\max}\hat A_{\max},
\qquad \forall\,\mathbf r\in S.
\end{equation}
Hence
\begin{align}
\|\hat u-u\|_{L^2(S)}^2
&\le
\|\hat A-A\|_{L^2(S)}^2
+
\nu^2 A_{\max}\hat A_{\max}
\int_S \Delta\tau(\mathbf r)^2\,d\mathbf r \\
&=
\|\hat A-A\|_{L^2(S)}^2
+
\nu^2 A_{\max}\hat A_{\max}\,
\|\Delta\tau\|_{L^2(S)}^2,
\end{align}
which proves \eqref{eq:regional_uniform_upper_bound}.
\end{proof}

\subsection{Remark on Scope}

The results above should be interpreted as structural characterizations of the regional field error under the local amplitude--phase abstraction used in Sec.~\ref{sec:asymmetry_transferability}. Proposition~\ref{prop:regional_exact_decomp} is exact and shows that amplitude mismatch and phase-sensitive mismatch enter the global prediction error through qualitatively different mechanisms. Corollary~\ref{cor:regional_upper_bound} further makes explicit that the phase-induced component carries a frequency-dependent scaling factor. Their role in the main text is explanatory rather than exhaustive: they formalize one mechanism behind the empirical asymmetry observed in Fig.~\ref{fig:em_similarity}, namely that phase-sensitive mismatch becomes increasingly costly in higher-frequency prediction, while coarse amplitude remains comparatively more stable and reusable across frequencies.

\section{Dataset Construction}
\label{app:dataset_details}

This appendix provides additional details on the benchmark construction and data splits used in the experiments. The three benchmarks instantiate the same lower-to-higher frequency prediction setting studied in the main text: a lower-frequency regime $\mathcal V_{\mathrm{LF}}$ is used to train the operator backbone, while a disjoint higher-frequency regime $\mathcal V_{\mathrm{HF}}$ is used for target-frequency adaptation and evaluation. Tables~\ref{tab:dataset_construction} and~\ref{tab:dataset_splits} summarize the corresponding dataset specifications and train/test splits.

\paragraph{SimpleWave benchmark.}
The SimpleWave benchmark is generated by a two-path wave simulator on a two-dimensional square domain $\Omega=[0,10]\times[0,10]$ in dimensionless spatial coordinates. The field is discretized on a $64\times64$ grid. Each sample is specified by an environment parameter $e=v$, where $v$ is a scalar propagation speed sampled uniformly from $[0.8,1.2]$, and by a frequency $\nu$, which is treated as a dimensionless spectral parameter in this benchmark. The model input is represented as two spatial channels, $[v,\nu]$, where both quantities are broadcast over the grid. The output is a scalar complex field $u(\mathbf r;v,\nu)\in\mathbb C^{64\times64}$.

The simulator combines a direct propagation path and a reflected path. Let $\mathbf r_s=(2.0,5.0)$ denote the source location and let $\mathbf r_m=(2.0,-5.0)$ denote the mirror source used to define the reflected path. For a spatial location $\mathbf r=(x,y)$, define
\begin{equation}
	R_1(\mathbf r)=\|\mathbf r-\mathbf r_s\|_2,
	\qquad
	R_2(\mathbf r)=\|\mathbf r-\mathbf r_m\|_2.
\end{equation}
The corresponding travel-time functions are
\begin{equation}
	\tau_1(\mathbf r;v)
	=
	\frac{R_1(\mathbf r)}{v}
	\left(1+\eta\,q_1(\mathbf r)\right),
	\qquad
	\tau_2(\mathbf r;v)
	=
	\frac{R_2(\mathbf r)+d}{v}
	\left(1+\eta\,q_2(\mathbf r)\right),
\end{equation}
where $q_1$ and $q_2$ are fixed smooth spatial perturbation fields, $\eta=0.08$ controls their strength, and $d=0.35$ is a delay bias for the reflected path. Specifically, with $L_x=L_y=10$, we use
\begin{align}
	q_1(\mathbf r)
	&=
	0.45\sin\left(\frac{2\pi x}{L_x}\right)
	+
	0.35\cos\left(\frac{2\pi y}{L_y}\right)
	+
	0.20\sin\left(\frac{2\pi(x+0.6y)}{L_x}\right),\\
	q_2(\mathbf r)
	&=
	0.40\cos\left(\frac{2\pi(x-0.3y)}{L_x}\right)
	+
	0.30\sin\left(\frac{2\pi y}{L_y}\right)
	+
	0.15\cos\left(\frac{2\pi(x+y)}{L_y}\right).
\end{align}
These perturbation fields are fixed across samples. The complex field is then generated as
\begin{equation}
	u(\mathbf r;v,\nu)
	=
	a_1(\mathbf r)
	\exp\left(i\,2\pi\nu\,\tau_1(\mathbf r;v)\right)
	+
	\lambda\,a_2(\mathbf r)
	\exp\left(i\,2\pi\nu\,\tau_2(\mathbf r;v)\right),
\end{equation}
with reflection coefficient $\lambda=0.18$.
Here the factor $2\pi$ reflects the cyclic-frequency convention used in the simulator; in the theoretical notation of Sec.~\ref{sec:2.2}, this constant factor is absorbed into the generic spectral variable $\nu$.
The two spatial amplitude envelopes are fixed deterministic functions of location:
\begin{equation}
	a_1(\mathbf r)
	=
	\frac{\exp(-\alpha_1 R_1(\mathbf r))}
	{\left(R_1(\mathbf r)+c_1\right)^{p_1}},
	\qquad
	a_2(\mathbf r)
	=
	\frac{\exp(-\alpha_2 R_2(\mathbf r))}
	{\sqrt{R_2(\mathbf r)+\epsilon}},
\end{equation}
where $\alpha_1=0$, $c_1=0.8$, $p_1=0.3$, $\alpha_2=0.12$, and $\epsilon$ is a small numerical constant. Thus, the amplitude envelopes are not learned and do not vary randomly across samples; sample-to-sample variation is induced by the sampled propagation speed $v$ and by the queried frequency $\nu$, while the path geometry and envelope parameters are fixed.

The lower-frequency regime is
\begin{equation}
	\mathcal V_{\mathrm{LF}}=\{1,2,3,4\},
\end{equation}
and the higher-frequency regime is
\begin{equation}
	\mathcal V_{\mathrm{HF}}=\{4.8,6,8\},
\end{equation}
where $\nu$ is treated as a dimensionless spectral parameter. These three higher-frequency values correspond to $20\%$, $50\%$, and $100\%$ extrapolation beyond the largest lower-frequency value, respectively. For each frequency, 250 samples are generated. The lower-frequency data are split into 800 training samples and 200 test samples, while the higher-frequency data are split into 150 training samples and 600 test samples.

\paragraph{Helmholtz benchmark.}
The Helmholtz benchmark is generated by solving a two-dimensional heterogeneous frequency-domain Helmholtz equation on a square domain $\Omega=[0,1]\times[0,1]$ in dimensionless spatial coordinates. The field is discretized on a $64\times64$ grid. For each sample, the complex field $u(\mathbf r;n,k)$ satisfies
\begin{equation}
	-\Delta u(\mathbf r)-k^2 n(\mathbf r)\bigl(1+i\sigma_{\mathrm{abs}}(\mathbf r)\bigr)u(\mathbf r)
	=
	s(\mathbf r).
\end{equation}
Here $n(\mathbf r)$ is the heterogeneous medium, $k$ is the wavenumber, $\sigma_{\mathrm{abs}}(\mathbf r)$ is a boundary sponge profile, and $s(\mathbf r)$ is a localized complex source. The medium is generated from a smooth Gaussian random field. Let $\xi(\mathbf r)$ denote a standard Gaussian noise field and let $\mathcal G_{\ell_g}$ denote Gaussian smoothing on the grid with smoothing width $\ell_g=8.0$ grid cells. We form a normalized smooth field
\begin{equation}
	g(\mathbf r)
	=
	\frac{\mathcal G_{\ell_g}\xi(\mathbf r)-\mu_g}{s_g},
\end{equation}
where $\mu_g$ and $s_g$ are the empirical mean and standard deviation after smoothing. The medium is then
\begin{equation}
	n(\mathbf r)
	=
	\operatorname{clip}\left(1+0.25\,g(\mathbf r),\,0.6,\,1.4\right).
\end{equation}
Thus, the sample-to-sample variability in the Helmholtz benchmark comes from the sampled medium field $n(\mathbf r)$ and the queried wavenumber $k$.

The source is a localized complex Gaussian centered at $\mathbf r_s=(0.30,0.50)$:
\begin{equation}
	s(\mathbf r)
	=
	A_s
	\exp\left(-\frac{\|\mathbf r-\mathbf r_s\|_2^2}{2\sigma_s^2}\right)
	\exp\left(i\beta\,\frac{(\mathbf r-\mathbf r_s)^\top \mathbf d_s}{\sigma_s}\right),
\end{equation}
where $A_s=1$, $\sigma_s=0.03$, $\beta=\pi/6$, and $\mathbf d_s=(\cos(\pi/4),\sin(\pi/4))$. The phase-ramp term introduces a mild directional bias near the source while keeping the source localized.

We apply a sponge profile near the boundary. Let $D_x(\mathbf r)$ and $D_y(\mathbf r)$ denote the distances from $\mathbf r$ to the nearest vertical and horizontal boundaries, respectively. With absorbing width $w_{\mathrm{abs}}=0.15$, strength $\gamma_{\mathrm{abs}}=1.5$, and power $p_{\mathrm{abs}}=2$, the profile is
\begin{equation}
	\sigma_{\mathrm{abs}}(\mathbf r)
	=
	\gamma_{\mathrm{abs}}
	\left[
	\left(\frac{w_{\mathrm{abs}}-D_x(\mathbf r)}{w_{\mathrm{abs}}}\right)_+^{p_{\mathrm{abs}}}
	+
	\left(\frac{w_{\mathrm{abs}}-D_y(\mathbf r)}{w_{\mathrm{abs}}}\right)_+^{p_{\mathrm{abs}}}
	\right],
\end{equation}
where $(\cdot)_+=\max(\cdot,0)$. The equation is solved by a second-order finite-difference discretization, yielding a sparse linear system for the interior grid values.

The model input is represented as two spatial channels, $[n,k]$, where $k$ is broadcast over the grid. The target is a scalar complex field $u(\mathbf r;n,k)\in\mathbb C^{64\times64}$. For this benchmark, the generic spectral variable $\nu$ corresponds to the wavenumber $k$ used in the finite-difference solver. The lower-frequency regime is
\begin{equation}
	\mathcal V_{\mathrm{LF}}=\{10,15,20,25\},
\end{equation}
and the higher-frequency regime is
\begin{equation}
	\mathcal V_{\mathrm{HF}}=\{30,37.5,50\},
\end{equation}
where $\nu$ is treated as a dimensionless spectral parameter. These three higher-frequency values correspond to $20\%$, $50\%$, and $100\%$ extrapolation beyond the largest lower-frequency value, respectively. For each spectral value, 500 samples are generated. The lower-frequency data are split into 1600 training samples and 400 test samples, while the higher-frequency data are split into 300 training samples and 1200 test samples.

\paragraph{Maxwell benchmark.}

The Maxwell benchmark is generated using full-wave electromagnetic simulations in CST Studio Suite. 
We employ the time-domain solver and place electric-field monitors at multiple prescribed frequencies to extract the corresponding frequency-domain field responses from a single broadband simulation. 
For each dielectric configuration, the simulated electric-field distribution is recorded on the observation plane at the selected frequencies, yielding complex-valued field samples for cross-frequency prediction.
The physical region spans $50\,\mathrm{cm}$ along each coordinate direction, and the learning problem uses two-dimensional $(y,z)$ observation slices sampled on an $80\times92$ grid. Each dielectric configuration consists of five horizontally stacked material layers, with relative permittivity increasing monotonically from the top layer to the bottom layer. The model input is represented by three spatial channels $[\varepsilon_r,x,f]$, where $\varepsilon_r(y,z)$ is the relative-permittivity distribution on the observation slice, $x$ is the slice coordinate broadcast over the grid, and $f$ is the excitation frequency broadcast over the grid.

The field is excited by a localized electric dipole source placed $1.5\,\mathrm{cm}$ above the top material surface and horizontally centered in the domain. Owing to its compact spatial support relative to the computational domain and sampling scale, the dipole is treated as a point-like excitation source. The upper and lateral boundaries use open-space radiation conditions, while the lower boundary uses an electric-wall condition,
\begin{equation}
	E_t = 0,
\end{equation}
where $E_t$ denotes the tangential electric field.

The target is the selected dominant polarization component
\begin{equation}
	u(y,z;\varepsilon_r,x,f)=E_y(y,z;f)\in\mathbb C^{80\times92},
\end{equation}
extracted from the CST frequency-domain solution. In this benchmark, the generic spectral variable $\nu$ corresponds to the physical frequency $f$ in GHz. The lower-frequency regime is
\begin{equation}
	\mathcal V_{\mathrm{LF}}
	=
	\{1.0,1.5,2.0,2.5\}\,\mathrm{GHz},
\end{equation}
and the higher-frequency regime is
\begin{equation}
	\mathcal V_{\mathrm{HF}}
	=
	\{3.0,3.75,5.0\}\,\mathrm{GHz}.
\end{equation}
These higher-frequency values correspond to $20\%$, $50\%$, and $100\%$ extrapolation beyond the largest lower-frequency value of $2.5\,\mathrm{GHz}$, respectively. We use 26 dielectric configurations and 30 observation slices for each frequency, giving 780 samples per frequency. The lower-frequency data are split into 2496 training samples and 624 test samples, while the higher-frequency data are split into 468 training samples and 1872 test samples.

\begin{table*}[t]
	\centering
	\small
	\setlength{\tabcolsep}{4pt}
	\renewcommand{\arraystretch}{1.08}
	\caption{Dataset construction and spectral splits used in the benchmark experiments.}
	\label{tab:dataset_construction}
	\resizebox{\textwidth}{!}{%
		\begin{tabular}{lcccccc}
			\toprule
			Dataset & Generator & Grid & Input channels & Target & LF spectral set & HF spectral set \\
			\midrule
			SimpleWave &
			two-path wave simulator &
			$64\times64$ &
			$[v,\nu]$ &
			scalar complex field &
			$\{1,2,3,4\}$ &
			$\{4.8,6,8\}$ \\
			Helmholtz &
			finite-difference Helmholtz solver &
			$64\times64$ &
			$[n,k]$ &
			scalar complex field &
			$\{10,15,20,25\}$ &
			$\{30,37.5,50\}$ \\
			Maxwell &
			time domain solver &
			$80\times92$ &
			$[\varepsilon_r,x,f]$ &
			$E_y$ complex field &
			$\{1.0,1.5,2.0,2.5\}$ GHz &
			$\{3.0,3.75,5.0\}$ GHz \\
			\bottomrule
		\end{tabular}
	}
\end{table*}

\begin{table*}[t]
	\centering
	\small
	\setlength{\tabcolsep}{5pt}
	\renewcommand{\arraystretch}{1.08}
	\caption{Train/test splits used in the benchmark experiments. The LF split is used for training the frozen operator backbone, while the HF split is used for target-frequency enhancement and evaluation. Joint splits are used by joint-training baselines.}
	\label{tab:dataset_splits}
	\begin{tabular}{lccccc}
		\toprule
		Dataset & LF total & LF train/test & HF total & HF train/test & Joint train/test \\
		\midrule
		SimpleWave & $4\times250=1000$ & $800/200$ & $3\times250=750$ & $150/600$ & $950/600$ \\
		Helmholtz & $4\times500=2000$ & $1600/400$ & $3\times500=1500$ & $300/1200$ & $1900/1200$ \\
		Maxwell & $4\times780=3120$ & $2496/624$ & $3\times780=2340$ & $468/1872$ & $2964/1872$ \\
		\bottomrule
	\end{tabular}
\end{table*}

\section{Training and Model Hyperparameters}
\label{app:training_details}
All experiments are conducted on a single NVIDIA RTX 4090 GPU with 24 GB memory.
All operator baselines use unit-Gaussian normalization for inputs and outputs, with complex fields represented in magnitude--phase form. The FNO baselines are trained for 200 epochs. SimpleWave uses a 4-layer FNO with 24 Fourier modes, hidden width 48, and lifting/projection width 128. Helmholtz uses a 4-layer FNO with 24 Fourier modes, hidden width 64, and lifting/projection width 256. Maxwell uses a 4-layer FNO with 32 Fourier modes, hidden width 64, and lifting/projection width 256. The HNO baselines are also trained for 200 epochs: SimpleWave uses width 8 with 24 Fourier modes and factor 0.75, Helmholtz uses width 11 with 24 Fourier modes and factor 0.75, and Maxwell uses width 11 with 32 Fourier modes and factor 0.75.

CFM-Joint and the second stage of APEX use the same U-Net architecture: base width 128, channel multipliers $[1,2,4]$, two residual blocks per resolution level, mid-block attention with 4 heads, and time embedding dimension 64. The CFM models are trained for 500 epochs on SimpleWave, 700 epochs on Helmholtz, and 1000 epochs on Maxwell. At inference time, all CFM models use midpoint ODE sampling with 50 steps.

CFM-Joint is trained on the joint LF+HF split and conditions on the dataset input channels. For SimpleWave and Helmholtz, these channels contain the environment field and spectral variable. For Maxwell, they contain the permittivity map $\epsilon$, the slice coordinate $x$, and the spectral variable.

APEX uses the same CFM architecture and training epochs, but augments the conditioning with the coarse amplitude anchor $a_{\mathrm{coarse}}$ and sine--cosine phase-prior maps $[\sin\phi_{\mathrm{base}},\cos\phi_{\mathrm{base}}]$, as defined in the main text. SimpleWave and Helmholtz use these spatial maps together with their environment fields. Maxwell follows the same interface and additionally includes the slice coordinate $x$ in the environment representation.

\section{Physics-Aware Instantiations of the Phase Prior}
\label{app:phase_prior_inst}

This appendix provides the concrete instantiations of the phase prior in Sec.~\ref{sec:3.3} for the experimental domains considered in this paper. The prior is intentionally lightweight: its role is to provide a coarse oscillatory scaffold consistent with source location, geometry, and an effective phase scale, rather than to reproduce the full target solution. In particular, the prior is used only through its angle in the main model, while transferable amplitude information is supplied separately by the coarse anchor.

The concrete instantiations used in the three benchmarks are summarized in Tab.~\ref{tab:domain_phase_prior}. For SimpleWave and Helmholtz, the prior retains only the dominant direct path, even though the SimpleWave data generator contains an additional reflected component. This choice keeps the prior deliberately simple and avoids reproducing the full simulator. For Maxwell, we retain a direct path together with one dominant reflected / mirrored path to capture the main geometric oscillatory pattern on the observation slice.

The quantities $v_{\mathrm{ref}}(e)$, $n_{\mathrm{ref}}(e)$, and $\epsilon_{\mathrm{ref}}(e)$ denote sample-wise effective propagation speed, Helmholtz medium coefficient, and permittivity, respectively. Here $L_m(\mathbf r,\mathbf r_s)$ denotes the geometric path length of the $m$-th retained path from source $\mathbf r_s$ to location $\mathbf r$. The carrier coefficient $k_0$ converts the benchmark-specific spectral variable into an effective path-accumulated phase scale, and $\kappa_{\mathrm{ref}}(e,\nu)$ then adapts this scale to the current sample through the corresponding medium summary.

\begin{table*}[h]
\centering
\small
\setlength{\tabcolsep}{4pt}
\renewcommand{\arraystretch}{1.1}
\caption{Physics-aware instantiations of the phase prior used in the experimental domains.}
\label{tab:domain_phase_prior}
\resizebox{\textwidth}{!}{%
\begin{tabular}{cccccc}
\toprule
Domain & Environment & $G_{\mathrm{prior}}$ & $\kappa_{\mathrm{ref}}(e,\nu)$ & $M$ & $a_m$ \\
\midrule
SimpleWave &
speed $v$ &
$\exp(i\,\kappa_{\mathrm{ref}}L_1)$ &
$k_0\,v_{\mathrm{ref}}^{-1}(e)$ &
1 &
$a_1=1$ \\
Helmholtz &
medium $n$ &
$\exp(i\,\kappa_{\mathrm{ref}}L_1)$ &
$k_0\,\sqrt{n_{\mathrm{ref}}(e)}$ &
1 &
$a_1=1$ \\
Maxwell &
medium $\epsilon$ &
$a_1\,\exp(i\,\kappa_{\mathrm{ref}}L_1)+a_2\,\exp(i\,\kappa_{\mathrm{ref}}L_2)$ &
$k_0\,\sqrt{\epsilon_{\mathrm{ref}}(e)}$ &
2 &
$a_1=1,\ a_2=1$ \\
\bottomrule
\end{tabular}
}
\end{table*}

\section{Additional Qualitative Results}
\label{app:additional_qualitative}

Figs.~\ref{fig:appendix_qualitative_hf20} and \ref{fig:appendix_qualitative_hf50} complement the HF100 comparison in Fig.~\ref{fig:hf100_qualitative} with smaller frequency gaps. At HF20, most baselines recover the coarse field layout more clearly than at HF100, but they still exhibit visible phase misalignment and local structural errors. APEX better preserves the dominant amplitude organization while maintaining more coherent phase patterns across the three benchmarks.

\begin{figure*}[h]
\centering
\includegraphics[width=\textwidth]{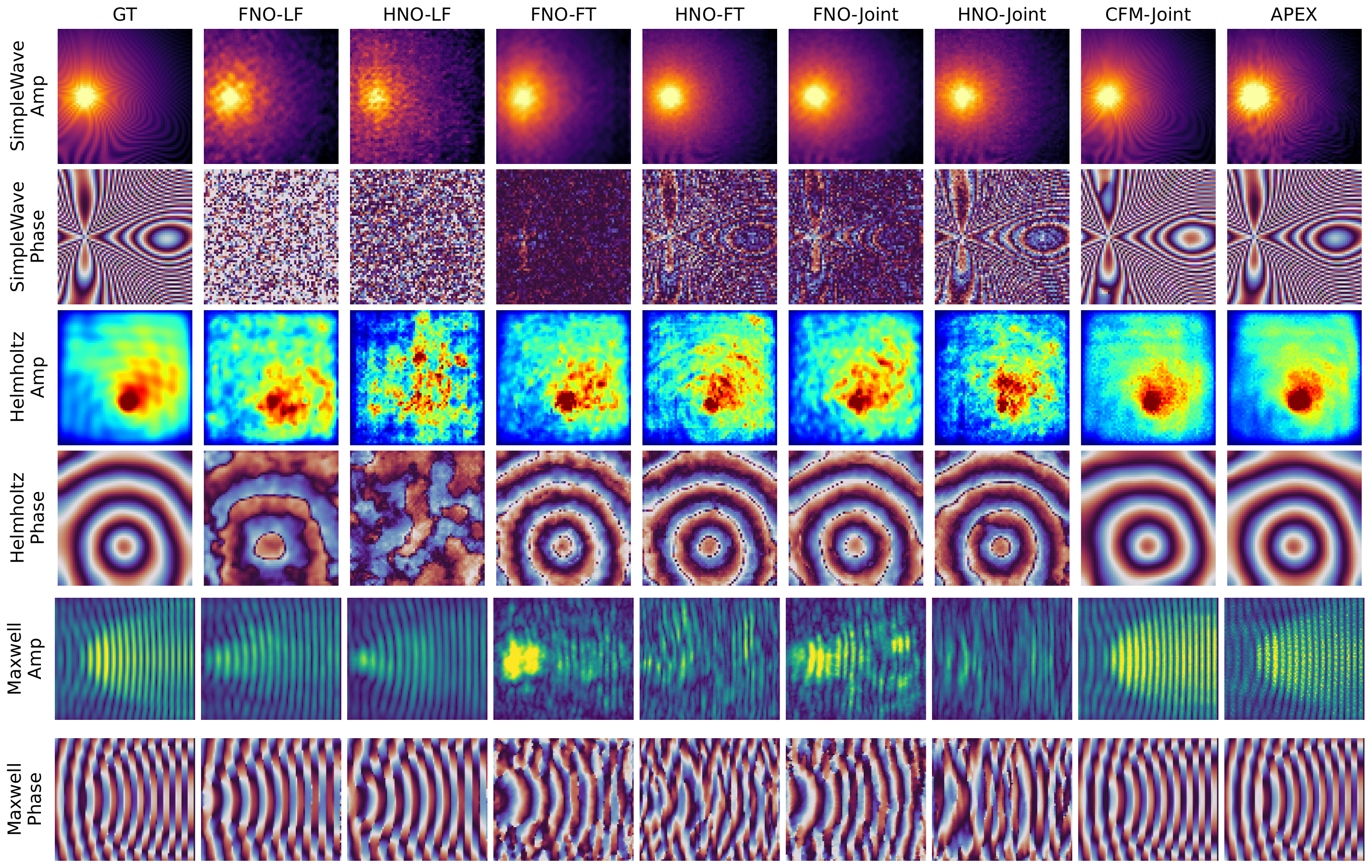}
\caption{Additional qualitative comparison at HF20.}
\label{fig:appendix_qualitative_hf20}
\end{figure*}

At HF50, the gap from the lower-frequency regime is larger and the difference between methods becomes more apparent. Direct extrapolation and jointly trained baselines often preserve plausible low-frequency structure but lose fine oscillatory alignment, whereas APEX remains closer to the ground truth in both amplitude localization and phase organization.

\begin{figure*}[h]
\centering
\includegraphics[width=\textwidth]{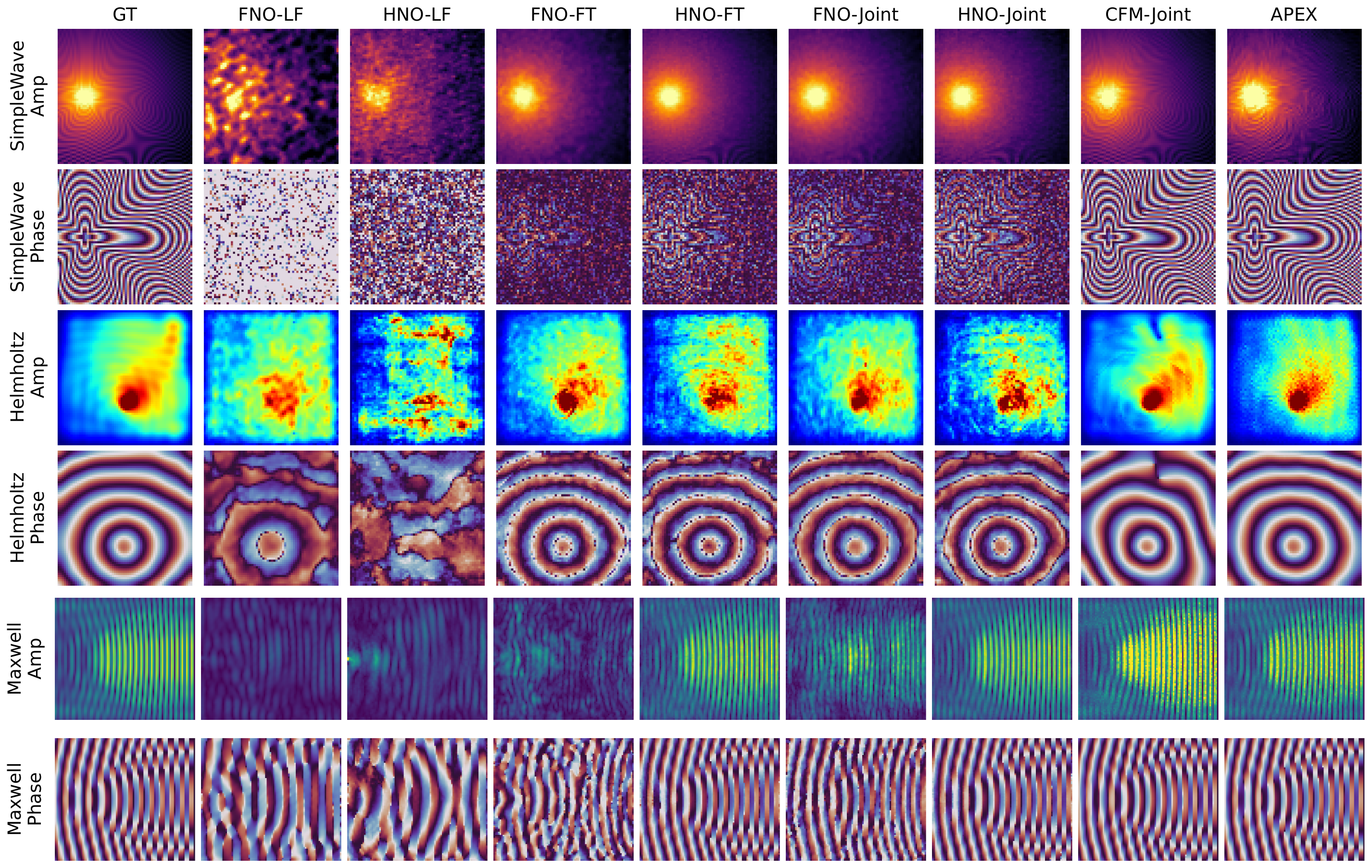}
\caption{Additional qualitative comparison at HF50.}
\label{fig:appendix_qualitative_hf50}
\end{figure*}

\section{Bootstrap Confidence Intervals for Main Results}
\label{app:bootstrap_main_results}

Table~\ref{tab:appendix-formal-multihf-bootstrap} complements Tab.~\ref{tab:main_results} by reporting uncertainty estimates for the higher-frequency prediction results. Each entry is shown as mean $\pm$ half-width of a 95\% non-parametric bootstrap confidence interval over the fixed HF test split, using $B{=}2000$ resamples. Following the main text, we report results for the three higher-frequency target groups, HF20, HF50, and HF100, on the SimpleWave, Helmholtz, and Maxwell benchmarks. For space reasons, the Overall summary column is omitted here. Lower $H^1$ and higher AWPC are better.

\begin{table*}[t]
\centering
\scriptsize
\setlength{\tabcolsep}{2.8pt}
\renewcommand{\arraystretch}{1.08}
\caption{Bootstrap confidence intervals for the higher-frequency prediction results across three benchmarks. Each target group reports mean $\pm$ half-width of a 95\% non-parametric bootstrap confidence interval for $H^1$/AWPC over the fixed HF test split ($B{=}2000$). Underlined entries denote the best result within each dataset and metric column.}
\label{tab:appendix-formal-multihf-bootstrap}
\begin{tabular}{lcccccc}
\toprule
\multicolumn{7}{c}{\textbf{SimpleWave}} \\
\midrule
\textbf{Method} & \multicolumn{2}{c}{\textbf{HF20}} & \multicolumn{2}{c}{\textbf{HF50}} & \multicolumn{2}{c}{\textbf{HF100}} \\
\cmidrule(lr){2-3}\cmidrule(lr){4-5}\cmidrule(lr){6-7}
& \textbf{$H^1 \downarrow$} & \textbf{AWPC$\uparrow$} & \textbf{$H^1 \downarrow$} & \textbf{AWPC$\uparrow$} & \textbf{$H^1 \downarrow$} & \textbf{AWPC$\uparrow$} \\
\midrule
FNO-LF & 1.377 $\pm$ 0.005 & 0.012 $\pm$ 0.001 & 1.454 $\pm$ 0.005 & 0.013 $\pm$ 0.001 & 1.429 $\pm$ 0.005 & 0.013 $\pm$ 0.001 \\
FNO-FT & 1.053 $\pm$ 0.010 & 0.157 $\pm$ 0.014 & 1.114 $\pm$ 0.018 & 0.238 $\pm$ 0.017 & 1.128 $\pm$ 0.011 & 0.133 $\pm$ 0.011 \\
HNO-LF & 1.359 $\pm$ 0.005 & 0.012 $\pm$ 0.001 & 1.420 $\pm$ 0.005 & 0.014 $\pm$ 0.001 & 1.401 $\pm$ 0.005 & 0.014 $\pm$ 0.001 \\
HNO-FT & 0.923 $\pm$ 0.032 & 0.496 $\pm$ 0.029 & 1.021 $\pm$ 0.033 & 0.465 $\pm$ 0.027 & 1.071 $\pm$ 0.031 & 0.400 $\pm$ 0.028 \\
FNO-Joint & 0.963 $\pm$ 0.026 & 0.403 $\pm$ 0.026 & 1.032 $\pm$ 0.033 & 0.438 $\pm$ 0.026 & 1.100 $\pm$ 0.027 & 0.324 $\pm$ 0.025 \\
HNO-Joint & 0.885 $\pm$ 0.038 & 0.540 $\pm$ 0.031 & 0.973 $\pm$ 0.039 & 0.531 $\pm$ 0.030 & 1.044 $\pm$ 0.034 & 0.433 $\pm$ 0.031 \\
CFM-Joint & 0.935 $\pm$ 0.048 & 0.620 $\pm$ 0.034 & 1.151 $\pm$ 0.046 & 0.516 $\pm$ 0.031 & 1.234 $\pm$ 0.030 & 0.351 $\pm$ 0.025 \\
APEX & \underline{0.396 $\pm$ 0.004} & \underline{0.909 $\pm$ 0.002} & \underline{0.494 $\pm$ 0.013} & \underline{0.865 $\pm$ 0.006} & \underline{0.804 $\pm$ 0.020} & \underline{0.659 $\pm$ 0.018} \\
\midrule
\multicolumn{7}{c}{\textbf{Helmholtz}} \\
\midrule
\textbf{Method} & \multicolumn{2}{c}{\textbf{HF20}} & \multicolumn{2}{c}{\textbf{HF50}} & \multicolumn{2}{c}{\textbf{HF100}} \\
\cmidrule(lr){2-3}\cmidrule(lr){4-5}\cmidrule(lr){6-7}
& \textbf{$H^1 \downarrow$} & \textbf{AWPC$\uparrow$} & \textbf{$H^1 \downarrow$} & \textbf{AWPC$\uparrow$} & \textbf{$H^1 \downarrow$} & \textbf{AWPC$\uparrow$} \\
\midrule
FNO-LF & 1.611 $\pm$ 0.008 & 0.169 $\pm$ 0.005 & 1.686 $\pm$ 0.008 & 0.098 $\pm$ 0.004 & 2.261 $\pm$ 0.029 & 0.050 $\pm$ 0.002 \\
FNO-FT & 0.762 $\pm$ 0.005 & 0.897 $\pm$ 0.002 & 0.783 $\pm$ 0.005 & 0.874 $\pm$ 0.003 & 0.966 $\pm$ 0.009 & 0.759 $\pm$ 0.005 \\
HNO-LF & 2.989 $\pm$ 0.027 & 0.097 $\pm$ 0.004 & 3.743 $\pm$ 0.041 & 0.042 $\pm$ 0.002 & 5.637 $\pm$ 0.097 & 0.032 $\pm$ 0.002 \\
HNO-FT & 0.756 $\pm$ 0.005 & 0.871 $\pm$ 0.002 & 0.909 $\pm$ 0.007 & 0.821 $\pm$ 0.003 & 1.276 $\pm$ 0.021 & 0.696 $\pm$ 0.006 \\
FNO-Joint & 0.715 $\pm$ 0.006 & 0.903 $\pm$ 0.002 & 0.779 $\pm$ 0.007 & 0.850 $\pm$ 0.003 & 0.899 $\pm$ 0.010 & 0.737 $\pm$ 0.006 \\
HNO-Joint & 0.930 $\pm$ 0.008 & 0.853 $\pm$ 0.002 & 1.006 $\pm$ 0.010 & 0.794 $\pm$ 0.003 & 1.401 $\pm$ 0.025 & 0.662 $\pm$ 0.005 \\
CFM-Joint & 0.852 $\pm$ 0.006 & 0.653 $\pm$ 0.014 & 0.895 $\pm$ 0.006 & 0.543 $\pm$ 0.015 & 0.930 $\pm$ 0.007 & 0.397 $\pm$ 0.016 \\
APEX & \underline{0.422 $\pm$ 0.010} & \underline{0.928 $\pm$ 0.003} & \underline{0.492 $\pm$ 0.013} & \underline{0.885 $\pm$ 0.005} & \underline{0.634 $\pm$ 0.019} & \underline{0.815 $\pm$ 0.008} \\
\midrule
\multicolumn{7}{c}{\textbf{Maxwell}} \\
\midrule
\textbf{Method} & \multicolumn{2}{c}{\textbf{HF20}} & \multicolumn{2}{c}{\textbf{HF50}} & \multicolumn{2}{c}{\textbf{HF100}} \\
\cmidrule(lr){2-3}\cmidrule(lr){4-5}\cmidrule(lr){6-7}
& \textbf{$H^1 \downarrow$} & \textbf{AWPC$\uparrow$} & \textbf{$H^1 \downarrow$} & \textbf{AWPC$\uparrow$} & \textbf{$H^1 \downarrow$} & \textbf{AWPC$\uparrow$} \\
\midrule
FNO-LF & 1.090 $\pm$ 0.009 & 0.066 $\pm$ 0.011 & 1.023 $\pm$ 0.002 & 0.003 $\pm$ 0.004 & 1.008 $\pm$ 0.000 & 0.000 $\pm$ 0.002 \\
FNO-FT & 0.840 $\pm$ 0.017 & 0.465 $\pm$ 0.020 & 0.899 $\pm$ 0.014 & 0.301 $\pm$ 0.022 & 0.954 $\pm$ 0.007 & 0.169 $\pm$ 0.013 \\
HNO-LF & 1.097 $\pm$ 0.007 & 0.044 $\pm$ 0.012 & 1.045 $\pm$ 0.003 & 0.012 $\pm$ 0.006 & 1.014 $\pm$ 0.001 & 0.002 $\pm$ 0.002 \\
HNO-FT & 0.912 $\pm$ 0.026 & 0.421 $\pm$ 0.030 & 0.947 $\pm$ 0.026 & 0.329 $\pm$ 0.032 & 0.962 $\pm$ 0.021 & 0.277 $\pm$ 0.031 \\
FNO-Joint & 0.904 $\pm$ 0.020 & 0.422 $\pm$ 0.024 & 0.933 $\pm$ 0.020 & 0.337 $\pm$ 0.025 & 0.959 $\pm$ 0.012 & 0.233 $\pm$ 0.018 \\
HNO-Joint & 0.842 $\pm$ 0.025 & 0.441 $\pm$ 0.031 & 0.900 $\pm$ 0.026 & 0.342 $\pm$ 0.032 & 0.948 $\pm$ 0.021 & 0.285 $\pm$ 0.031 \\
CFM-Joint & 0.701 $\pm$ 0.029 & \underline{0.828 $\pm$ 0.018} & 0.897 $\pm$ 0.033 & 0.680 $\pm$ 0.027 & 1.187 $\pm$ 0.034 & 0.411 $\pm$ 0.036 \\
APEX & \underline{0.697 $\pm$ 0.034} & 0.812 $\pm$ 0.022 & \underline{0.790 $\pm$ 0.031} & \underline{0.725 $\pm$ 0.022} & \underline{0.929 $\pm$ 0.029} & \underline{0.624 $\pm$ 0.026} \\
\bottomrule
\end{tabular}
\end{table*}

\section{Limitations}
\label{app:limitations}
A current limitation of APEX is its inference cost. Since the proposed enhancer is based on conditional flow matching, generating a higher-frequency field requires iterative sampling and is therefore slower than a single forward pass of a neural operator. In this work, we focus on improving prediction accuracy under scarce target-frequency supervision, and do not optimize the sampling efficiency in detail.

Another limitation is that we instantiate the lower-frequency backbone using FNO in the main APEX framework. The proposed amplitude-anchor design is not restricted to FNO in principle, since any lower-frequency surrogate that produces a coarse target-frequency prediction can be used to provide the amplitude anchor. A systematic evaluation of alternative neural-operator backbones and physics-aware surrogate architectures is left for future work.

Finally, the Green's-function-inspired phase prior used in this work is intentionally simplified. It provides source-aware and geometry-aware oscillatory cues, but does not fully model complex multipath effects, strong dispersion, or detailed boundary interactions. Future work may incorporate more accurate physics-based priors or data-calibrated phase models. 

\section{Broader Impact}
\label{app:broader_impact}

This work studies higher-frequency prediction for complex wave fields under limited target-frequency supervision. The proposed method may benefit scientific and engineering applications that require efficient wave-field prediction, such as acoustics, electromagnetics, and related simulation-driven design tasks. By reducing the amount of higher-frequency supervision required for adaptation, the method has the potential to lower computational and data acquisition costs in frequency-domain modeling.

We do not foresee direct negative societal impacts from this work. The study is methodological and is evaluated on controlled wave-field benchmarks. However, as with other surrogate modeling approaches, predictions should be used with caution in safety-critical applications, where additional validation against high-fidelity solvers, measurements, and domain-specific constraints is necessary.


\end{document}